\documentclass{iopjournal}
\usepackage{amssymb}
\usepackage{amsmath}
\usepackage{amsfonts}
\usepackage{amsthm}
\usepackage{bm}
\usepackage{booktabs}
\usepackage[english]{babel}
\usepackage{xcolor}
\usepackage{graphicx}
\usepackage{subcaption}
\usepackage{caption}
\usepackage{cleveref}
\newtheorem{definition}{Definition}[section]
\newtheorem{theorem}{Theorem}[section]

\begin{document}

\articletype{Paper} %	 e.g. Paper, Letter, Topical Review...

\title{Overparameterized Multiple Linear Regression\\ as Hyper-Curve Fitting}

\author{Elisa Atza$^1$\orcid{0009-0002-5483-7271} and Neil Budko$^{1,*}$\orcid{0000-0002-9649-7386}}

\affil{$^1$Numerical Analysis, DIAM, EEMCS, Delft University of Technology, Delft, The Netherlands}

\affil{$^*$Author to whom any correspondence should be addressed.}

\email{n.v.budko@tudelft.nl}

% REQUIRED

%%%%%%%%%%%%%%%%%%%%%%%%%%%%%%%%%%%%%%%%%%%%%%%%%%%%%%%%%%%%%%%%%%%%%%%%%%%%%%%%
% REQUIRED
\begin{abstract}
This work demonstrates that applying a fixed-effect multiple linear regression (MLR) model to an overparameterized dataset is mathematically equivalent to fitting a hyper-curve parameterized by a single scalar. This reformulation shifts the focus from global coefficients to individual predictors, allowing each to be modeled as a function of a common parameter. We prove that this overparameterized linear framework can yield exact predictions even when the underlying data contains nonlinear dependencies that violate classical linear assumptions. By employing parameterization in terms of the dependent variable and a monomial basis, we validate this approach on both synthetic and experimental datasets. Our results show that the hyper-curve perspective provides a robust framework for regularizing problems with noisy predictors and offers a systematic method for identifying and removing 'improper' predictors that degrade model generalizability.
\end{abstract}

\section{Introduction}
Many models encountered in practical applications of Machine Learning (ML) are overparameterized. Some are superficially so, e.g., a linear random-effect model \cite{MEM2022} with many predictors that come from the same probability distribution, which by itself is described by a few unknown parameters only. Others, like a linear fixed-effect model, also known as the Multiple Linear Regression (MLR) model, with fewer training samples than unknowns, may happen to be truly overparameterized. While there is no universally accepted definition of an overparameterized ML model, research in this area has recently uncovered several interesting phenomena, such as the double-dipping of the prediction error \cite{holzmuller2020universality}, and the so-called benign overfitting \cite{BenignOverfitting2020}, \cite{BenignOverfitting2022}. 

The problems where the number of predictors is extremely large are common in many applications. Recent technological developments in chemistry, biology, medicine and agriculture have allowed for high-throughput data acquisition pipelines resulting in relatively large amounts of predictor variables compared to the feasible number of experiments in which a target phenotype or a trait is measured. For instance, in chemometric research, where one seeks to predict a physical or biological property from the infrared spectrum of the substance, the number of spectral components is in the order of thousands \cite{Chemomet2005} \cite{Chemomet2023}. In metabolomics, the biological traits are predicted from the relative abundance of small-molecule chemicals in a biological sample, the number of metabolites can reach tens of thousands \cite{metabol2010}, \cite{metabo2018}. Similar numbers of predictors are used in the microbiome-based predictions \cite{Micro2020}. Finally, in genomics, the number of Single-Nucleotide Polymorphisms (SNP's), that may potentially predict the phenotype of some living organism, can be as high as tens of millions  \cite{Genom2021}, \cite{Genom2023}. At the same time, the duration and costs of experiments aimed at measuring the dependent variables (phenotype, traits, etc.) are often much higher. Therefore, the number of such experiments, i.e., the number of training samples, is typically a fraction of the number of predictors (features) -- hundreds or a few thousands at most \cite{BioPred2013}, \cite{BioPred2015}, \cite{BioPred2016}, thus making the overparameterized nature of such problems inescapable.

Due to the enormous success of the Artificial Neural Networks (ANN's) in image classification, the current focus in the omics-related ML is on the application of these sophisticated nonlinear models to the existing and new data. However, truly deep ANN's are also overparameterized, which leads to learning problems when the number of training samples is low and require a large training/validation dataset to tune hyperparameters and avoid overfitting. Sometimes this problem is circumvented by creating an artificial augmented training dataset \cite{SSDNN2021}. Therefore, it is not surprising that whenever the training is performed on small datasets, the quality of predictions obtained with an ANN is only marginally better, if at all, when compared to the predictions with the simple overparameterized MLR model \cite{ANNvsMLR2022}, \cite{DNNvsLR2020}. 

One could argue that the good results by the overparameterized MLR are also due to overfitting or a poor testing procedure (e.g. data leakage). While this certainly may be the case in some practical studies, generally, the fixed-effect MLR model with fewer training samples than predictors is well-understood on a theoretical level \cite{vogel2002},\cite{Izenman2008}. The solution to the resulting underdetermined linear system, if it exists, is not unique. As the next best thing, one aims at recovering the minimum-norm least-squares (LS) solution, which always exists and is unique. The minimum-norm LS solution may, however, be sensitive to noise in the dependent variable. If the level of noise in the dependent variable is known, then there is a regularized solution, which minimizes the error with respect to the noise-free minimum-norm LS solution. If, as it often happens in practice, the level of noise in the dependent variable is not known, then the level of regularization can either be estimated from the data with the cross-validation method or a similar technique, or has to be chosen subjectively. The case of the noise in predictor variables is much less understood, but there exist various errors-in-variables models \cite{TLS2002}.

The main focus of the present paper is the adequacy, performance and optimization of overparameterized linear models. Specifically, the authors aimed at understanding the nature of overparameterized datasets, identifying the circumstances where the MLR model makes good or bad predictions of such datasets, improving the interpretability of regression results beyond the traditional feature weights, and increasing the prediction accuracy, simultaneously making the model more adequate, i.e., satisfying the linear assumptions. This has lead us to the column-centered reformulation of the MLR, where the (inverse) relation between each predictor variable and the dependent variable can to a large extent be analyzed independently of other predictor variables. We show that the predictions made by such an Inverse Regression (IR) model are identical to the predictions of the MLR model on a class of overparameterized datasets. Topologically this means that on such datasets the MLR model is not a hyper-plane as suggested by its mathematical form, but a hyper-curve, parameterizable by a single scalar parameter, which, for the sake of interpretability can be chosen as the dependent variable.

Due to its variable-centered nature, the hyper-curve approach allows to filter out the predictors (features) that are either too noisy or do not satisfy the topological requirements of a linear model, which significantly improves the predictive power of the trained linear model, removes the features that may otherwise introduce the illusion of understanding \cite{IllusionsPaper}, and suggests the subsets of predictors where a non-linear or a higher-dimensional-manifold model would be more appropriate.

The paper is structured as follows. In \Cref{sec:PARCUR} we define Fundamentally OverParameterized (FOP) datasets, MLR, PARametric hyper-CURve (PARCUR) and IR models, prove their equivalence and identify the condition for the exact prediction of a test dataset. In \Cref{sec:ModelErrors} we study the behavior of the polynomial IR model applied to a dataset which contains both the polynomial predictors as well as predictors that have a non-functional relation to the dependent variable. We establish conditions under which such a dataset is a FOP dataset. In \Cref{sec:Regularization} we consider noisy data and introduce a polynomial degree truncation regularization scheme that can handle the noise in both the dependent and predictor variables. In \Cref{sec:Features} we propose a novel predictor removal algorithm that does not suffer from the ambiguities common to heuristic feature selection methods. In \Cref{sec:yarn} we apply the regularized IR model with predictor removal to the widely available experimental chemometric Yarn dataset \cite{plsRpackage}, \cite{YarnPaper}, and demonstrate the presence of both curve-like and higher-dimensional manifolds in this dataset. Finally, we present our conclusions and discuss the possible extensions of the PARCUR and IR models.

\section{Parametric Hyper-Curve and Inverse Regression models}
\label{sec:PARCUR}
In this and the following section we focus on the mathematical properties of the model and all data is assumed to be exact. Data with additive noise will be considered in \cref{sec:Regularization}. The standard multiple regression model expresses the dependent variable $y$ as a function of $p$ predictor variables $x_{j}$, $j=1,\dots,p$, i.e.,
\begin{align}
\label{eq:ForwardModel}
y = f(x_{1},\dots,x_{p}),\;\;\;f: {\mathbb R}^{p}\rightarrow{\mathbb R}.
\end{align}
In particular, the MLR model,
\begin{align}
\label{eq:MLRM}
y = \sum_{j=1}^{p}\beta_{j}x_{j},
\end{align}
obviously, belongs to this class of models. Topologically, the MLR equation (\ref{eq:MLRM}) describes a $p$-dimensional linear object, a hyper-plane, in a $(p+1)$-dimensional space. 

Now, consider the model of the form:
\begin{align}
\label{eq:ParametricCurveModel}
\begin{split}
y & = x_{0}(s), 
\\
x_{j} & = x_{j}(s),\;\;\;j=1,\dots,p;
\\
s&\in[a,b]\subset{\mathbb R},
\end{split}
\end{align}
which simply states that all data are considered to be the functions of some scalar parameter $s$.
The equations (\ref{eq:ParametricCurveModel}), describe a parametric hyper-curve in a $(p+1)$-dimensional space, essentially, a one-dimensional object. We shall call this model the PARametric hyper-CURve (PARCUR) model.

Under a monotone transformation of variables, the PARCUR model is equivalent to the Inverse Regression (IR) model:
\begin{align}
\label{eq:IRM}
\begin{split}
x_{j} & = x_{j}(y),\;\;\;j=1,\dots,p.
\end{split}
\end{align}
In the inverse relation (\ref{eq:IRM}) the independent variables $x_{j}$, $j=1,\dots,p$, are considered to be the functions of the dependent variable $y$. This model naturally emerges in the context of calibration problems \cite{Calibration1994}. It is also the most easily interpretable version of the PARCUR model as the functions $x_{j}(y)$ provide an insight into the change of each individual predictor variable as a function of the dependent variable $y$, if such a functional dependence exists.

Obviously, the general model (\ref{eq:ForwardModel}) and the IR model (\ref{eq:IRM}) are completely equivalent only under very stringent constraints on the function $f$. Specifically, for a complete equivalence the function $f(x_{1},\dots,x_{p})$ as well as all individual functions $x_{j}(y)$ should be invertible. While an inverse of the MLR model (\ref{eq:MLRM}) in the form (\ref{eq:IRM}) may exist on certain subsets of ${\mathbb R}^{p}$, a general nonlinear function $f$ in (\ref{eq:ForwardModel}) is not invertible and neither is a general IR or PARCUR model. Yet, the predictive power of both the MLR and the PARCUR models trained on a finite training dataset of a certain general type appears to be the same even when they are not mutually invertible.

Our main results, expressed in \cref{thm:ExactPredictionMLRM} and \cref{thm:Equivalence}, are the conditions on the exact prediction by the MLR model and the equivalence of the predictions made by the MLR and PARCUR models for what we call a FOP dataset. This equivalence allows to analyze the different types of column functions $x_{j}(s)$, $j=1,\dots,p$, and establish the conditions on the existence of a FOP dataset for different types of predictor data.

In practice, any regression model is trained on a finite discrete dataset. An overparameterized dataset arises whenever the number $n$ of training samples is smaller than the number $p$ of parameters or predictors. This can happen simply due to the lack of experimental data and additional experiments could transform an incidentally overparameterized dataset into a well-defined or even an underparameterized one. However, in the present paper, we shall focus on the more fundamental case, where the simple addition of the training data does not change the overparameterized nature of the dataset. To do so, we imagine that our dataset is potentially infinite, i.e., one can draw or generate data from this dataset as many times as one wishes, e.g., by simply performing the measurements of the same physical quantities over and over again.

\begin{definition}
\label{def:Dataset}
The dataset
\begin{align*}
    {\mathcal S}=\{(y_i, x_{i,1},\dots,x_{i,p} ), \; y_i,x_{i,j} \in \mathbb{R}|\; i=1, \dots; j=1,\dots,p\}
\end{align*}
is a {\it fundamentally overparameterized} (in the linear sense) dataset of rank $q$, if, for any $m$, the data-matrix $S_{m}\in{\mathbb R}^{m\times (p+1)}$ with its rows from ${\mathcal S}$, i.e.,
\begin{align*}
S_{m}&= [{\bm y}, X],\;\;\;
{\bm y} = 
\begin{bmatrix}
y_{1}\\
\vdots\\
y_{m}
\end{bmatrix},
\;\;\;
X = 
\begin{bmatrix}
x_{1,1}& \dots & x_{1,p}\\
\vdots & \ddots & \vdots\\
x_{m,1}& \dots & x_{m,p}
\end{bmatrix},
\end{align*}
has the property: 
\begin{align}
\label{eq:Overparameterized}
{\rm rank}(S_{m}) \leq q \leq p.
\end{align}
A {\it training} dataset is any fixed-size data matrix $S_{n}\in{\mathbb R}^{n\times(p+1)}$ with the rows from the FOP dataset ${\mathcal S}$. A training dataset is {\it complete} if ${\rm rank}(S_{n})=q$.  The complement dataset ${\mathcal S}_{\rm t} = {\mathcal S}\setminus {\mathcal S}_{n}$, is called the {\it test} dataset.
\end{definition}

In practice we are dealing with arbitrary but fixed-size testing datasets as well. To summarize, the dependent-variable data from $S_{n}$ and ${\mathcal S}_{\rm t}$ will be stored in the data-vectors ${\bm y}\in{\mathbb R}^{n}$ and ${\bm y}_{\rm t}\in{\mathbb R}^{m}$, and the corresponding predictor-variable data will be stored in the data-matrices $X\in{\mathbb R}^{n\times p}$ and $X_{\rm t}\in{\mathbb R}^{m\times p}$, respectively.

An overparameterized MLR model (\ref{eq:MLRM}) corresponds to an underdetermined linear algebraic system $X{\bm \beta}={\bm y}$, which is usually solved in the minimum-norm least-squares sense as $\hat{\bm \beta} = X^{T}(XX^{T})^{-1}{\bm y}$, where we have assumed that ${\rm rank}(X)=n$. Then, such a `trained' MLR model will make the following predictions for the testing dataset:
\begin{align}
\label{eq:YandXPredictionMLRM}
\begin{split}
\hat{\bm y}_{\rm t} & = X_{\rm t}X^{T}(XX^{T})^{-1}{\bm y},
\\
\hat{X}_{\rm t} & = X_{\rm t}X^{T}(XX^{T})^{-1}X.
\end{split}
\end{align}
Both ${\bm y}$ and $X$ are considered to be given in the training dataset $S_{n}$, whereas only the predictor data-matrix $X_{\rm t}$ is considered to be given in the testing dataset ${\mathcal S}_{\rm t}$. Hence, strictly speaking, the prediction $\hat{X}_{\rm t}$ of the given matrix $X_{\rm t}$ is not necessary. However, the fact that it represents a projection of the rows of $X_{\rm t}$ on the row space of the training matrix $X$ will be used in our subsequent analysis. Moreover, the two predictions (\ref{eq:YandXPredictionMLRM}) can now be written as the prediction $\hat{S}_{\rm t}$ of the testing dataset matrix $S_{\rm t}$:
\begin{align}
\label{eq:SPredictionMLRM}
\hat{S}_{\rm t} = X_{\rm t}X^{T}(XX^{T})^{-1}S_{n}.
\end{align}

The following \Cref{def:ColumnFunctionSpace} establishes the connection between the continuous world of predictor variables as described by the predictor functions and the discrete world of datasets, i.e., the columns of the data matrices.
\begin{definition}
\label{def:ColumnFunctionSpace}
A linear vector space ${\mathcal V}_{n}$, with ${\rm dim}({\mathcal V}_{n})=n$, is called a {\it column function space} of the dataset ${\mathcal S}$ of rank $n$ if there exist $p+1$ {\it column functions} $x_{j}(s)\in {\mathcal V}_{n}$, $x_{j}: {\mathbb R}\rightarrow{\mathbb R}$, $j=0,\dots p$, such that:
\begin{align}
\label{eq:ColumnFunctions}
\begin{split}
y_{i} &= x_{0}(s_{i}),
\\
x_{i,j} &= x_{j}(s_{i}),\;\;j=1,\dots,p,
\end{split}
\end{align}
for any data-point $(y_{i},x_{i,1},\dots,x_{i,p})$ from ${\mathcal S}$.
\end{definition}

Notice that we use the same notation for the predictor data $x_{i,j}$ and the column functions $x_{j}(s)$, with the relation being $x_{i,j}=x_{j}(s_{i})$. Also, obviously, $s=y$ in the IR model (\ref{eq:IRM}).

Let $\{v_{k}:{\mathbb R}\rightarrow{\mathbb R}\vert\,k=1,\dots,n\}$ be a basis of the column function space ${\mathcal V}_{n}$ of the dataset ${\mathcal S}$ of rank $n$. Then, each column function $x_{j}(s)$ can be expanded as:
\begin{align}
\label{eq:ColumnFunctionExpansion}
x_{j}(s) = \sum_{k=1}^{n}a_{k,j}v_{k}(s).
\end{align}
The entries of the {\it basis matrix} $V\in{\mathbb R}^{n\times n}$ are the sampled basis functions $v_{k}(s_{i})$, $k=1,\dots,n$, $i=1,\dots,n$:
\begin{align}
\label{eq:BasisMatrix}
V = 
\begin{bmatrix}
v_{1}(s_{1}) & \dots & v_{n}(s_{1})\\
\vdots & \ddots & \vdots \\
v_{1}(s_{n}) & \dots & v_{n}(s_{n})
\end{bmatrix}.
\end{align}
If $V$ is invertible, then the data-vector ${\bm y}$ and the data-matrix $X$ of the training set ${S}_{n}$ can be decomposed as follows:
\begin{align}
\label{eq:Xdecomposition}
X = VA,\;\;\;{\bm y} = V{\bm a}_{0},
\end{align}
where the elements $[A]_{k,j}=a_{k,j}$ of the matrix $A\in{\mathbb R}^{n\times p}$ and the elements $[{\bm a}_{0}]_{k}=a_{k,0}$ of the vector ${\bm a}_{0}\in{\mathbb R}^{n}$ are the expansion coefficients from (\ref{eq:ColumnFunctionExpansion}).

Formally, the PARCUR models (\ref{eq:ParametricCurveModel}) and (\ref{eq:IRM}) can be trained by computing the matrix $A$ and the vector ${\bm a}_{0}$ as:
\begin{align}
\label{eq:TrainingIRM}
A = V^{-1}X,\;\;\;{\bm a}_{0} = V^{-1}{\bm y}.
\end{align}
To predict the test dataset ${\mathcal S}_{\rm t}$, one has to solve the following minimization problem:
\begin{align}
\label{eq:MinimizationIRM}
\hat{V}_{\rm t} = {\rm arg}\min_{V_{\rm t}\in{\mathbb R}^{m\times n}}\Vert X_{\rm t} - V_{\rm t}A \Vert^{2}_{2},
\end{align}
where $X_{\rm t}$ is the given predictor test data-matrix. If ${\rm rank}(A)={\rm rank}(S_{n})=n$, the solution of the problem (\ref{eq:MinimizationIRM}) is given by:
\begin{align}
\label{eq:LSVt}
\hat{V}_{\rm t} = X_{\rm t}A^{T}(AA^{T})^{-1}.
\end{align}
Then, applying the relations (\ref{eq:Xdecomposition}), the predictions of the PARCUR model can be written as follows:
\begin{align}
\label{eq:YandXPredictionIRM}
\begin{split}
\hat{\bm y}_{\rm t} & = \hat{V}_{\rm t}{\bm a}_{0} = X_{\rm t}A^{T}(AA^{T})^{-1}{\bm a}_{0},
\\
\hat{X}_{\rm t} & = \hat{V}_{\rm t}A = X_{\rm t}A^{T}(AA^{T})^{-1}A.
\end{split}
\end{align}
This can also be combined into the prediction $\hat{S}_{\rm t}$ of the testing dataset matrix $S_{\rm t}$:
\begin{align}
\label{eq:SPredictionIRM}
\hat{S}_{\rm t} = X_{\rm t}A^{T}(AA^{T})^{-1}A_{n},
\end{align}
where $A_{n}=[{\bm a}_{0}, A]$ and, from (\ref{eq:Xdecomposition}), $S_{n}=VA_{n}$.

The following \Cref{thm:ExactPredictionMLRM} establishes the condition under which the prediction made by the MLR is going to be exact.
\begin{theorem}
\label{thm:ExactPredictionMLRM}
The prediction $\hat{S}_{\rm t}$ produced by the MLR model is exact for any data matrix $S_{\rm t}$ from the fundamentally overparameterized dataset ${\mathcal S}$ of rank $q$ if and only if the training dataset ${S}_{q}$ is complete and ${\rm rank}(X)=q$.
\end{theorem}
\begin{proof} Let $S_{q}=[{\bm y},X]$ be a complete training set of the FOP dataset ${\mathcal S}$. By \cref{def:Dataset}, for any test set ${\mathcal S}_{\rm t}\subset {\mathcal S}$ with the data matrix $S_{\rm t}=[{\bm y}_{\rm t}, X_{\rm t}]$ there exists the matrix $U$ such that $S_{\rm t}=US_{q}$ and $X_{t}=UX$. Hence, from (\ref{eq:SPredictionMLRM}),
\begin{align}
\label{eq:eq:ProofExact1}
\hat{S}_{\rm t} = X_{\rm t}X^{T}(XX^{T})^{-1}S_{q} = UXX^{T}(XX^{T})^{-1}S_{q} = US_{q}=S_{\rm t}.
\end{align}
If the training set ${S}_{q}$ is not complete, then there exists a data row ${\bm s}_{\rm t}\in {\mathcal S}$ that is not a linear combination of the rows of $S_{q}$.
\end{proof}

Whether the training dataset is compete or not, the MLR and the PARCUR models are equivalent in the following sense.

\begin{theorem}
\label{thm:Equivalence}
Let ${S}_{n}=[{\bm y},X]$, ${\rm rank}(X)=n$, and ${\mathcal S}_{\rm t}=[{\bm y}_{\rm t},X_{\rm t}]$ be a training and a testing datasets of the fundamentally overparameterized dataset ${\mathcal S}$ of rank $q$, $n\leq q$. Let also ${\mathcal V}_{q}$ be the column function space of ${\mathcal S}$. Then, for any basis $\{v_{j}\}$ in ${\mathcal V}_{q}$ with the invertible basis matrix $V\in{\mathbb R}^{n\times n}$, the predictions of the testing dataset (\ref{eq:YandXPredictionMLRM}) and (\ref{eq:YandXPredictionIRM}) made, respectively, by the MLR and the PARCUR models are equal.
\end{theorem}
\begin{proof} The proof follows from substituting the relations (\ref{eq:TrainingIRM}) into (\ref{eq:YandXPredictionIRM}).
\end{proof}

If the training dataset is incomplete and the rows of the testing data-matrix $S_{\rm t}$ are not (yet) in the span of the rows of the training data-matrix $S_{n}$, there will be an error in the predictions made by the MLR and PARCUR models. Let $S_{\rm t}$ be decomposed as
\begin{align}
\label{eq:StDecomposion}
S_{\rm t} = US_{n}+S_{\rm t}^{\perp},\;\;\;S_{\rm t}^{\perp}S_{n}^{T} = {\bm y}_{\rm t}^{\perp}{\bm y}^{T}+X_{\rm t}^{\perp}X^{T}=0_{k,n},
\end{align}
where $S_{\rm t}^{\perp}=[{\bm y}_{\rm t}^{\perp}, X_{\rm t}^{\perp}]$, and $0_{k,n}\in{\mathbb R}^{k\times n}$ is the matrix of all zeros. The existence of this decomposition stems from the fact that ${\rm rank}(S_{n})=n$, i.e., $S_{n}S_{n}^{T}$ is invertible. 

Then, $X_{\rm t} = UX+X_{\rm t}^{\perp}$, and the prediction error will be:
\begin{align}
\label{eq:IncompleteSnPredictionError}
\begin{split}
S_{\rm t}-\hat{S}_{\rm t} & = US_{n}+S_{\rm t}^{\perp} - (UX+X_{\rm t}^{\perp})X^{T}(XX^{T})^{-1}S_{n}
\\
&=S_{\rm t}^{\perp}+X_{\rm t}^{\perp}X^{T}(XX^{T})^{-1}S_{n} = S_{\rm t}^{\perp}-{\bm y}_{\rm t}^{\perp}{\bm y}^{T}(XX^{T})^{-1}S_{n}
\\
&=\left[{\bm y}_{\rm t}^{\perp}-{\bm y}_{\rm t}^{\perp}{\bm y}^{T}(XX^{T})^{-1}{\bm y},\;X_{\rm t}^{\perp}-{\bm y}_{\rm t}^{\perp}{\bm y}^{T}(XX^{T})^{-1}X\right]
.
\end{split}
\end{align}
Hence, apart from the obvious case ${\bm y}_{\rm t}^{\perp}={\bm 0}$, the prediction of the dependent variable will also be exact if ${\bm y}^{T}(XX^{T})^{-1}{\bm y} = 1$.

Technically, the difference between the MLR and PARCUR models can be expressed as a simple transformation of the columns of the training dataset $S_{n}$ by the basis matrix $V$. The relative advantage of the IR version, $s=y$, of the PARCUR model over other choices of the parameter $s$ stems from the direct interpretation of the coefficient matrix $A$ since it provides an insight into the dependence of each predictor variable on the dependent variable. The magnitude and the sign of the coefficients in $A$ reflect the type (e.g., linear, quadratic, etc) and the strength of these dependencies. It is also clear why interpreting the ${\bm \beta}$ vector of the MLR model might be more problematic, as its relation to the coefficient matrix, $\hat{\bm \beta} = A^{T}(AA^{T})^{-1}{\bm a}_{0}$, is rather convoluted. In the standard MLR formulation (\ref{eq:MLRM}), a component $\hat{\beta}_{j}$ of the vector $\hat{\bm \beta}$ is interpreted as the weight of the additive contribution by the corresponding predictor $x_{j}$. However, it does not further specify the nature of the mathematical relation between $y$ and $x_{j}$.

It is also possible, and sometimes profitable, to train the PARCUR model on an {\it overcomplete} overparameterized training data set $S_{n}\in{\mathbb R}^{n\times{(p+1)}}$, where ${\rm rank}(S_{n})=q\leq p$, but $n>q$. However, since the basis matrix $V\in{\mathbb R}^{n\times q}$, ${\rm rank}(V)=q$, is now singular, one has to resort to the Ordinary Least-Squares (OLS) solution of the training problem:
\begin{align}
\label{eq:LSTraining}
\hat{A}_{n} = (V^{T}V)^{-1}V^{T}S_{n}.
\end{align}
In that case, $\hat{A}=(V^{T}V)^{-1}V^{T}X$, and the prediction of the test set makes use of these OLS estimates as $\hat{S}_{\rm t} = X_{\rm t}\hat{A}^{T}(\hat{A}\hat{A}^{T})^{-1}\hat{A}_{n}$. Although, the equivalence to the MLR model is only achieved if the chosen basis matrix $V$ coincides with the matrix of the left singular vectors of~$X$.

\section{Data containing polynomial column functions}
\label{sec:ModelErrors}
\Cref{thm:ExactPredictionMLRM} shows that the prediction by the MLR model is exact if the training set is complete. From the \Cref{def:Dataset} it is clear that a complete dataset is a subset of a FOP dataset, such that its data matrix has the maximal possible rank $q$, which can be smaller than the number of predictors $p$. Given a finite training dataset of size $n$ it is hard to tell if: a) it is a subset of a FOP dataset with some $q<p$, b) it is a complete dataset, i.e. $n=q$. In this section, working under the assumption that the data set ${\mathcal S}_{m}$ does not contain statistical noise, we establish the sufficient condition for the existence of a FOP dataset. To formulate these conditions it is convenient to make a choice of the column function space, see \Cref{def:ColumnFunctionSpace}. Here, we consider the polynomial function space and the {\it monomial} basis, which lead to easily interpretable regression results. 

We note a well-known fact that the monomial basis $\{v_{k}(y)=y^{k}\,\vert\, k=0,\dots,n-1\}$ for the column function space ${\mathcal V}_{n}$ will produce an invertible Vandermonde basis matrix $V$, given by
\begin{align}
\label{eq:Vandermonde}
V = 
\begin{bmatrix}
1 & y_{1} & y_{1}^{2}&\dots& y_{1}^{n-1}\\
1 & y_{2} & y_{2}^{2}&\dots& y_{2}^{n-1}\\
\vdots & \vdots & \cdots&\vdots\\
1 & y_{n} & y_{n}^{2}&\dots& y_{n}^{n-1}\\
\end{bmatrix},
\end{align}
if all entries of the dependent variable data-vector ${\bm y}=[y_{1},\dots,y_{n}]^{T}$ are distinct. In this basis, the entries $a_{i,j}$ of the coefficient matrix $A$ represent the coefficients of the polynomial functions $x_{j}(y)$ that may generate some of the columns of the predictor data-matrix $X$:
\begin{align}
\label{eq:PolynomialX}
x_{j}(y) = a_{1,j}+a_{2,j}y+\dots+a_{n,j}y^{n-1}.
\end{align}
This also puts the PARCUR model into the context of polynomial fitting. Whether the columns of $X$ have or have not been generated by polynomial functions of $y$, the IR model with the basis matrix (\ref{eq:Vandermonde}) will be projecting all columns of $X$ on the monomial basis. From the computational point of view, Vandermonde matrices, while theoretically invertible, are hard to work with for sizes above $n=15$ and become the source of significant round-off errors. Luckily, one normally does not need polynomial functions of very high degree to adequately describe a data column. To minimize the numerical errors, we also normalize the range of the $y$-data to fit within the interval $[-1,1]$.

In general, it is difficult to decide whether the training dataset is complete by simply inspecting the entries of its data matrix $S_{n}$. In theory, one could compute the Singular-Value Decomposition (SVD) of the matrix and see if the zero singular value appears after adding any new sample (row) to the training dataset. However, here we are interested in an arbitrary column function space ${\mathcal V}_{q}$ with the square invertible basis matrix $V\in{\mathbb R}^{n\times n}$, and the conclusions about the eventual completeness of ${S}_{n}$ will be based on the shape of the corresponding coefficient matrix $A_{n}=V^{-1}S_{n}$. 

Each of the $p+1$ data columns $x_{j}$ in a dataset ${\mathcal S}$ belongs to one of the three classes: 
\begin{enumerate}
\item{$x_{j}(y)$ is a polynomial in $y$ of degree less or equal to $n-1$}
\item{$x_{j}(y)$ is a polynomial in $y$ of degree higher than $n-1$}
\item{there is a non-functional dependence between $x_{j}$ and $y$}
\end{enumerate}
In particular, the first column $x_{0} = y$, i.e., the dependent variable, obviously, belongs to the first class with $n=2$.

A non-functional dependence between the predictor $x_{j}$ and the dependent variable $y$ would emerge if the data was situated on a curve that cannot be parameterized by $y$, \Cref{fig:nn-func-3d-example} (left). Choosing a different parameterization could transform these 'nonfunctional' data to functions $y(s)$, $x_{j}(s)$, $j=1,\dots,p$, as in the general PARCUR formulation. A more severe case of non-functional dependence arises where the data is situated on a higher-dimensional manifold, such as a hyper-surface, \Cref{fig:nn-func-3d-example} (right), and no alternative parameterization can fix this problem. The column data that one observes in such `nonfunctional' cases are illustrated in \Cref{fig:nn-func-3d-example} (bottom, two-dimensional scatter plots).
\begin{figure}
    \centering
    \hspace*{-0.8cm}
    \includegraphics[width=1.1\linewidth]{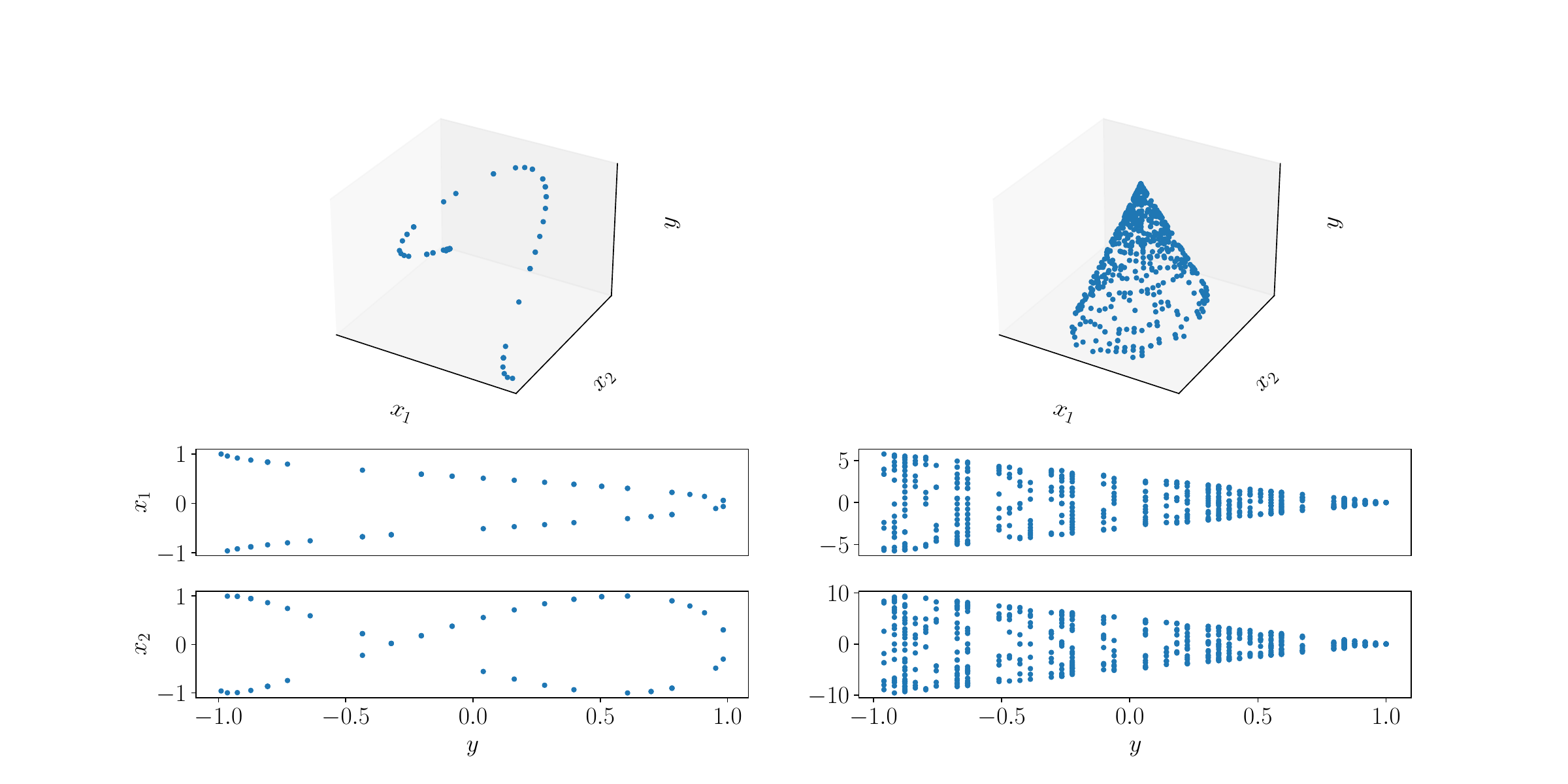}
    \caption{Examples of column data produced by non-functional relationships between the dependent variable $y$ and the predictor variables $x_{1}$ and $x_{2}$. Left: data on a curve that cannot be parameterized by $y$. Right: data on a conical surface. Two-dimensional scatter-plots: $x_{1}$ and $x_{2}$ column data sorted by $y$ and displayed as `functions' of $y$.}
    \label{fig:nn-func-3d-example}
\end{figure}

\begin{theorem}
\label{thm:polynomialFOP}
Let ${\mathcal S}$ be a dataset with one independent variable $y$ and $p$ predictor variables $x_{j}$, $j=1,\dots,p$. Let also ${\mathcal S}$ contain $k\leq p$ predictors that are polynomials in $y$ of degrees greater than $q-1$, or have a non-functional relation to $y$. Then, ${\mathcal S}$ is a fundamentally overparameterized dataset (in linear sense) of rank $q$, $2\leq q\leq p$, if its remaining $p-k$ predictors are the polynomials in $y$ of degree $r$, $1\leq r\leq q-k-1$.
\end{theorem}
\begin{proof}
Without the loss of generality we may assume that for any subset of size $m$ of the dataset ${\mathcal S}$, the dependent variable data ${\bm y}$ contains only the distinct values of $y$ so that the square monomial basis matrix $V\in{\mathbb R}^{m\times m}$ is invertible. Then, any data-matrix $S_{m}=[{\bm y},X]\in{\mathbb R}^{m\times(p+1)}$ can be represented as $S_{m}=VA_{m}$. By the conditions of the Theorem, for any $m>q$, subject to column reordering, the coefficient matrix $A_{m}\in{\mathbb R}^{m\times(p+1)}$ has the structure:
\begin{align}
A_{m}=
\begin{bmatrix}
{\bm e}_{2}&A_{1,1}& A_{1,2}\\
{\bm 0}& O&A_{2,2}
\end{bmatrix},\;\;\; A_{1,1}\in{\mathbb R}^{(q-k)\times(p-k)}, \;\;\;A_{1,2}\in{\mathbb R}^{(q-k)\times k}, \;\;\; A_{2,2}\in{\mathbb R}^{(m-q+k)\times k},
\end{align}
where $V^{-1}{\bm y}={\bm e}_{2}\in{\mathbb R}^{m}$ is the second standard basis vector and $O\in{\mathbb R}^{(m-q+k)\times(p-k)}$ is the matrix of all zeros. Here we have used the fact that the polynomial fit to non-functional data may produce a polynomial of degree greater than $q-1$. It is obvious that ${\rm rank}(A_{1,1})\leq (q-k)$ and, for any $m\geq (q-k)$, ${\rm rank}(A_{2,2})\leq k$. Therefore, ${\rm rank}(A) \leq q$, and ${\rm rank}(S_{m})\leq q$ for any $m$, showing that ${\mathcal S}$ is a FOP dataset.
\end{proof}

The above \Cref{thm:polynomialFOP} shows that having extremely high-degree polynomials, e.g., with degrees higher than $p-1$ and non-functional dependencies among the predictors does not prevent the MLR and the IR models from making exact predictions as long as there are also polynomial data of sufficiently low degree and the set on which the model is trained is complete. Moreover, a complete dataset can, in principle, be achieved with $n<p$ samples. The latter fact may seem surprising as it appears that we are able to recover a polynomial of degree higher than $n-1$ or a non-functional dependence by training on just $n$ data points. However, it becomes less surprising if we consider the form of the MLR and IR predictors given by (\ref{eq:YandXPredictionMLRM}) and (\ref{eq:YandXPredictionIRM}), as in both cases the leftmost matrix $X_{\rm t}$ contains the $X$-data from the test dataset which ones is trying to `predict'.

In the limiting case with $p$ high-degree polynomials and/or non-functional dependencies, a complete dataset will only be achieved with $n=p$ samples. It seems to be a waste of time and resources, though, to collect so much training data knowing that the majority of predictors do not even satisfy the model assumptions. We come back to this question in \Cref{sec:Features}. In the other limiting case, where all predictors are polynomials, the size of the complete training dataset can be as small as $n=2$ if the maximal degree of all polynomials is at most one and there is at least one predictor which is a linear function of $y$.

\Cref{fig:Error_Ex_X_y} (top) illustrates the performance of the IR model with the three classes of column data discussed above. Specifically, we are considering the cases where: all predictors are  polynomial functions in $y$ of degree $r-1\leq q$ (blue, circles), some of the predictors are high-degree polynomials in $y$ (orange, squares), and some of the predictors have non-functional relation to the dependent variable $y$ (green, triangles).
\begin{figure}[t]
    \centering
    \includegraphics[width=0.9\linewidth,   trim= {2cm 2.4cm 2cm 2cm},clip]{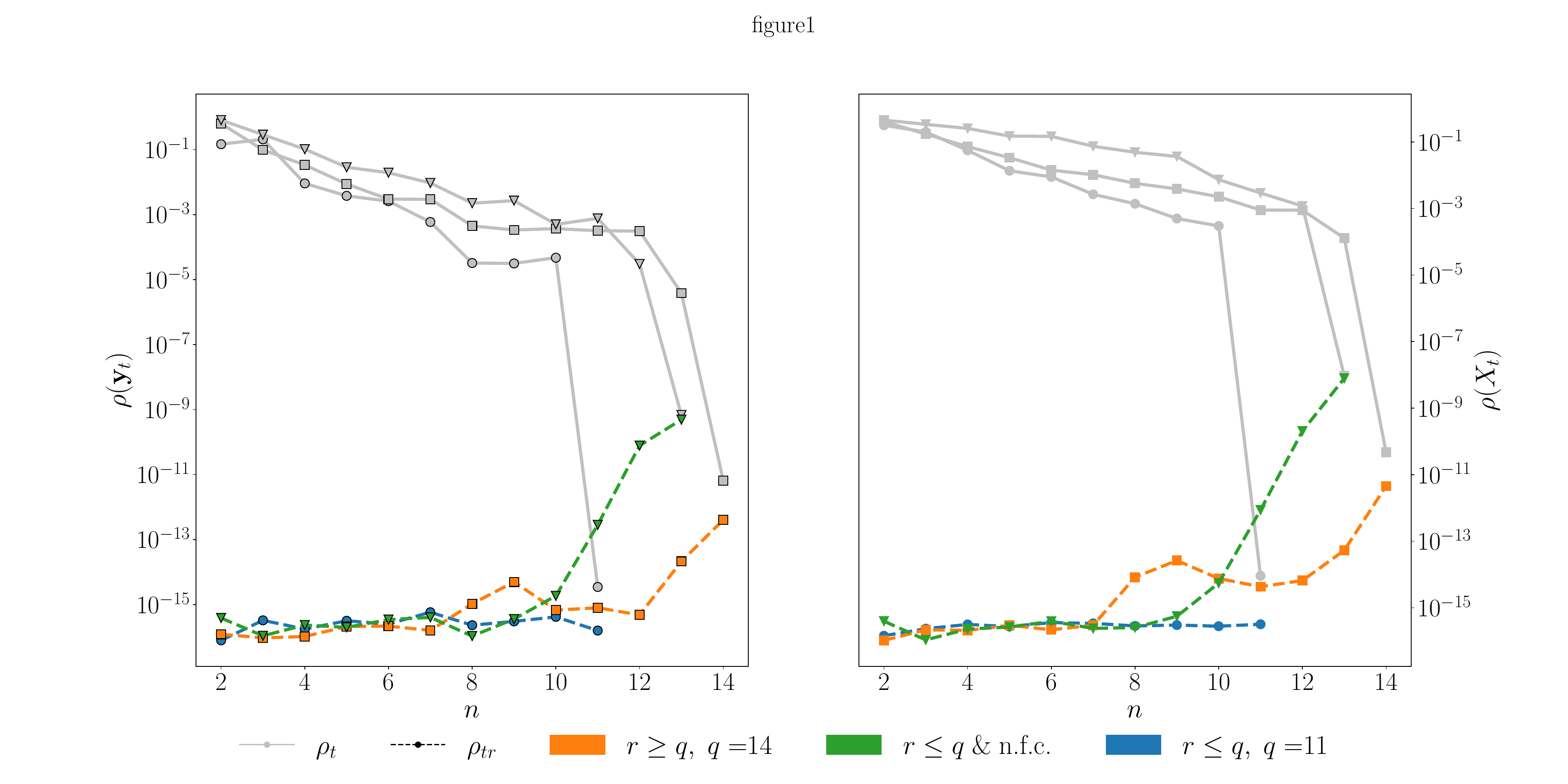}
    \includegraphics[width=0.9\linewidth,   trim= {2cm 0cm 2cm 2cm},clip]{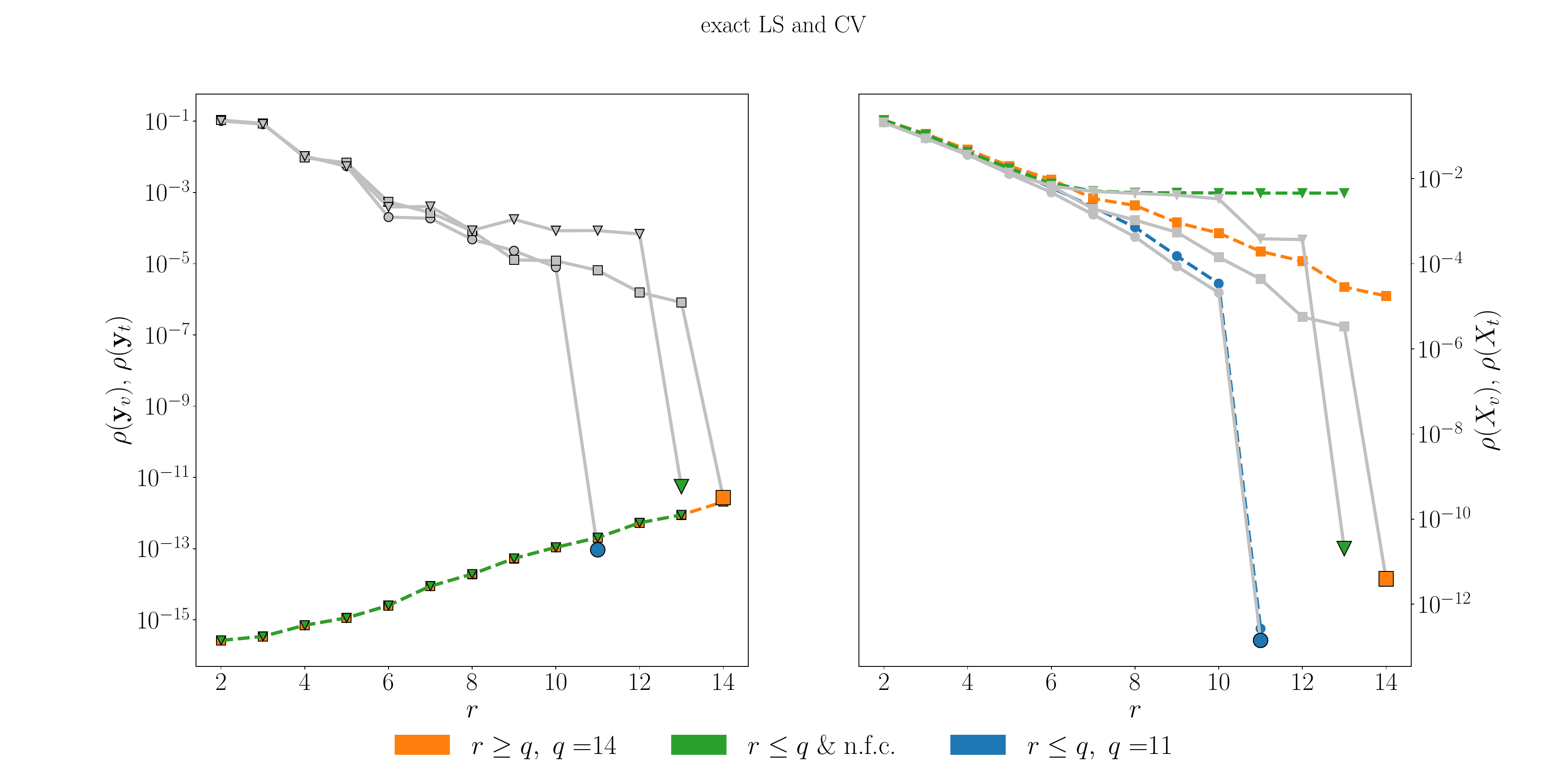}
    \caption{\small Drop in prediction errors of IR models when the training set becomes complete (noiseless data, left -- dependent variable ${\bm y}$, right -- independent predictor variables $X$). Colored lines -- training set errors, gray lines -- test set errors. Blue, circles: maximum polynomial degree of predictors is $11$. Orange, squares: maximum polynomial degree of predictors is $14$. Green, triangles: some of the predictors have non-functional dependence on the dependent variable.
    Top: square, invertible basis matrix $V$. Bottom: non-square basis matrix $V$ (over-complete training dataset) and a formal test of the polynomial truncation regularization algorithm with noiseless data. The identified optimal polynomial degrees $r^{*}$ coincide with the ranks of the corresponding complete training datasets (bright enlarged colored markers show the minima of validation errors).}
    \label{fig:Error_Ex_X_y}
\end{figure}

In these numerical experiments the data $x_{j}(y)$ are generated by randomly sampling the range of $y\in[-1,1]$ and evaluating polynomial functions of various degrees at the sampled points. Columns that are not functions of $y$ are generated as non-invertible functions $y(x_{j})$, similar to those in the examples of \Cref{fig:nn-func-3d-example}. The predictions are computed as in Eq.'s~(\ref{eq:YandXPredictionIRM}), where the coefficient matrix $A$ is obtained from the training $X$-data as in Eq.~(\ref{eq:TrainingIRM}).

The ranks of the complete datasets are: $q=11$ ($p=200$ polynomials of degree $r\leq 10$), $q=14$ ($p=203$, with $200$ polynomials of degree $r\leq 10$ and three polynomials of degree $r=17$), and $q=13$ ($p=202$, with $200$ polynomials of degree $r\leq 10$ and two non-functional predictors). In all three cases, we expect the prediction errors to vanish as soon as the rank of the training dataset reaches $q$.

The prediction errors in ${\bm y}$ and $X$ are measured on both the training and the test sets as follows:
\begin{align}
\label{rhot}
\begin{split}
    \rho({\bm y}_{\rm v})&=\frac{\Vert \bm y_{\rm v} - \hat{\bm y}_{\rm v}\Vert_2}{\Vert\bm y_{\rm v}\Vert_2},\;\;\;
    \rho(X_{\rm v})=\frac{\Vert X_{\rm v} - \hat{X}_{\rm v}\Vert_F}{\Vert X_{\rm v}\Vert_F},
\\
    \rho({\bm y}_{\rm t})&=\frac{\Vert \bm y_{\rm t} - \hat{\bm y}_{\rm t}\Vert_2}{\Vert\bm y_{\rm t}\Vert_2},\;\;\;
    \rho(X_{\rm t})=\frac{\Vert X_{\rm t} - \hat{X}_{\rm t}\Vert_F}{\Vert X_{\rm t}\Vert_F},
\end{split}
\end{align}
where ${\bm y}_{\rm v}$ and $X_{\rm v}$ are the training (validation) set data and ${\bm y}_{\rm t}$ and $X_{\rm t}$ are the test set data.

\Cref{fig:Error_Ex_X_y} shows the prediction errors as functions of the training set rank (top row), and of the polynomial representation degree (bottom row). The bright colored dashed lines give the training set errors and the dim solid gray lines show the corresponding errors on the test dataset. Plots on the left show the dependent variable errors and on the right -- the errors for the predictor variables. As expected, the errors drop as soon as the training dataset becomes complete. With the high-degree polynomial predictors and especially with non-functional predictors the errors on the training set begin to rise at the end and the errors on the test set do not drop to machine precision values as happens with the low-degree polynomial predictors. This is due to the fact that the coefficient matrices in the former two cases become ill-conditioned, making the predictors more sensitive to numerical round-off errors.

\section{Polynomial regularization}
\label{sec:Regularization}
All measured data are either random in nature or contain statistical noise. In the latter case, instead of the exact values $y_{i}$ of the dependent variable $y$, one usually measures the quantity $y_{i}+\epsilon_{i}$, where the `noise' $\epsilon_{i}$ is a single realization of a random variable $\epsilon$ with the distribution from some well-defined class, e.g., $\epsilon\sim{\mathcal N}(0,\sigma^{2})$, see e.g. \cite{Izenman2008}. The parameters of the distribution, such as the variance $\sigma^{2}$, are often unknown and estimated from the data. Being measured quantities, the independent variables $x_{j}$ may also be contaminated by noise. In the noisy setting, the PARCUR model becomes:
\begin{align}
\label{eq:ParametricCurveModelNoise}
\begin{split}
y & = y(s)+\epsilon_{0}, 
\\
x_{j} & = x_{j}(s)+\epsilon_{j},\;\;\;j=1,\dots,p;
\\
s&\in[a,b]\subset{\mathbb R}.
\end{split}
\end{align}
For simplicity we assume here the most common statistical hypotheses $\epsilon_{j}\sim{\mathcal N}(0,\sigma_{j}^{2})$, $j=0,\dots,p$, which provides an adequate description of the ``standard'' errors due to the finite sample size. We investigate the effect of the additive Gaussian noise $\epsilon_{j}\sim{\mathcal N}(0,\sigma_{j}^{2})$, $j=0,\dots,p$ applied in the following three ways: 
\begin{enumerate}
\item{noise only in the dependent variable $y$: $\sigma_{0}=\sigma\neq 0$, $\sigma_{j}=0$, $j=1,\dots,p$}
\item{independent and identically distributed (i.i.d.) noise in $S_{n}$: $\sigma_{j}=\sigma\neq 0$, $j=1,\dots,p$}
\item{independent, but not identically distributed noise in $S_{n}$: 
\begin{align}
\label{eq:NoiseStructure}
\sigma_{j}=\begin{cases}
\sigma\neq 0,& \text{if}\;\;j=0\;\;\text{and}\;\;x_{j}\in\{\text{noisy predictors}\}\\
0, &\text{if}\;\;j=0\;\;\text{and}\;\;x_{j}\in\{\text{exact predictors}\}
\end{cases}
\end{align}
}
\end{enumerate}
Obviously, the first `classical' case belongs to the last class if considered over the complete data matrix $S_{n}$.

To avoid overfitting and mitigate the effect of noise on the model predictions, one can use any of the standard regularization techniques, such as the Tikhonov regularization or the truncated SVD. In the case of the IR model, truncating the number of terms of the monomial basis used to represent the column functions appears to be the most natural regularization approach.

In the previous \Cref{sec:PARCUR}, the exact training data ${\bm y}$, could simply be sorted by magnitude and used as the values of the curve parameter to construct the monomial basis matrix $V$. While this is still permitted in the noisy case (we can choose the parameter $s$ as we wish), the interpretation of the coefficient matrix $A$ is no longer straightforward, if instead of ${\bm y}$  we are using ${\bm y}+{\bm \epsilon}$. 

At the risk of losing some of the interpretability of the coefficient matrix $\hat{A}$, we shall nevertheless be sorting all our  data by the magnitude of ${\bm y}+{\bm \epsilon}$. Notice, that the vector ${\bm y}+{\bm \epsilon}$ sorted by the (inaccessible) exact data ${\bm y}$ represents a smooth function (irregularly sampled $y$) with an additive `white' noise. Sometimes, sorting by the exact ${\bm y}$ can be achieved if the predictors contain a `clean' (noiseless) monotonous function of $y$ or even a noisy monotonous function that has been sorted in correct order. In that case, all columns in $S_{n}$ should be sorted by the column corresponding to this predictor variable.

If ${\bm y}+{\bm \epsilon}$ is correctly sorted by ${\bm y}$, then a multitude of denoising techniques is available for such smooth noisy signals, e.g., the Wiener filter. However, numerical experiments indicate that, in our case, the Wiener filter becomes effective starting from approximately $n=150$ data samples and is of little use below that threshold. Similar lower bound seems to hold for the effectiveness of the cross-validation and similar techniques that allow to deduce the optimal value of the regularization parameter (maximal degree $r^{*}$ of the monomial basis functions) from the training data. Therefore, we propose a regularization approach that depends on the number of available training samples, see \cref{tab:RegularizationScheme}. 
\begin{table}[t!]
    \centering
    \begin{tabular}{|c|c|c|}
        \hline
         & $n< 150$ & $n\geq 150$ \\
        \hline
        Step~1 & \multicolumn{2}{c|}{Sort ${\bm y}+{\bm \epsilon}$ (if possible by ${\bm y}$)}\\
        \hline
        Step~2 & -- & Apply Wiener filter on sorted ${\bm y}+{\bm \epsilon}$\\
        \hline
        Step~3 & Sort all columns by sorted ${\bm y}+{\bm \epsilon}$ & Sort all columns by sorted and filtered ${\bm y}+{\bm \epsilon}$\\
        \hline
        Step~4 & Choose $r^{*}$ subjectively &Use CV to find $r^{*}$\\
         \hline
    \end{tabular}
    \caption{Regularization strategies for the IR model with truncated monomial basis, depending on the number of training samples.}
    \label{tab:RegularizationScheme}
\end{table}

The optimally regularized representation $\hat{\bm x}_{j}(r^{*})$ of a noisy data-column ${\bm x}_{j}+{\bm \epsilon}_{j}$ minimizes the error between the exact (noiseless) column ${\bm x}_{j}$ and its regularized representation $\hat{\bm x}_{j}(r^{*})$, e.g. $\Vert{\bm x}_{j}-\hat{\bm x}_{j}(r^{*})\Vert_{2}$, where $r^{*}$ is the truncation index in terms of the monomial basis. 

Since the noiseless column is not known, the optimal truncation index $r^{*}$ is determined either subjectively by observing the quality of fit for several columns ($n<150$) or by the Cross Validation (CV) technique ($n\geq 150$). 

We employ a $10$-fold CV, where during each fold the model is trained on a randomly chosen subset of the training dataset and evaluated on a complementary (validation) subset.  As a metric, we consider the following validation errors:
\begin{align}
\label{eq:RhoV}
    \rho({\bm y}_{\rm v}) = \left\langle \frac{\|\bm{y}_{\rm v} - \hat{\bm y}_{\rm v}(r)\|_2}{\|\bm{y}_{\rm v}\|_2}\right\rangle,\;\;\;\;\;\;
    \rho(X_{\rm v}) = \left\langle \frac{\|X_{\rm v} - \hat{X}_{\rm v}(r)\|_{\rm F}}{\|X_{\rm v}\|_{\rm F}}\right\rangle,
\end{align}
where the angular brackets denote the arithmetic averaging over the CV folds. The optimal truncation index $r^{*}$ corresponds to the polynomial degree for which the minimum of $\rho(X_{\rm v})$ is attained. 

There is a curious reason behind the fact that we have to use the error $\rho(X_{\rm v})$ in the $X$-data rather than the usual error $\rho({\bm y}_{\rm v})$ in the ${\bm y}$-data. In the polynomial IR method, the vector of the dependent variable (either exact or noisy) is the second column of the basis Vandermonde matrix $V$, see (\ref{eq:Vandermonde}). Therefore, in exact arithmetic, the training ${\bm y}$-data is exactly reproduced for any $r\geq 2$. Hence, strictly speaking, the IR model always over-fits the training ${\bm y}$-data.
The regularization of the IR model is validated and tuned on the columns of the matrix $X$. This is possible, since apart from the noise in $X$ itself, the ${\bm y}$-noise is always propagated into the Vandermonde matrix $V$ and then into the columns of the coefficient matrix $\hat{A}=(V^{T}V)^{-1}V^{T}X$. Thus, the prediction $\hat{X}_{\rm t}=X_{\rm t}\hat{A}^{T}(\hat{A}\hat{A}^{T})^{-1}\hat{A}X$ will be always affected by noise. Simply put, using the noisy data ${\bm y}+{\bm \epsilon}$ to construct $V$ is equivalent to using a wrong parameterization $s\neq y$, whereas the column data are generated in the $s=y$ parameterization.

Finally, it is important to realize that there are two ways to regularize the noisy data. One, which is the focus of the present section, is to choose a single optimal (maximal) degree $r^{*}$ for the representation of all data columns, i.e., both ${\bm y}$ and all ${\bm x}_{j}$, $j=1,\dots,p$. This, obviously, has its drawbacks, since some of the columns may be noiseless or contain a different level of noise. A more flexible and precise way is to determine the optimal degree $r_{j}^{*}$, $j=0,1,\dots,p$, for each column individually. Our preliminary numerical experiments have shown that the latter `flexible' regularization approach does not necessarily result in a better prediction of the test ${\bm y}_{\rm t}$ data. Nonetheless, in our opinion, this flexible regularization deserves a separate in-depth investigation in the case of the non-i.i.d noise and for the purposes described in the next \Cref{sec:Features}.

\Cref{fig:Error_Ex_X_y} (bottom) illustrates the testing of the polynomial-degree truncation regularization scheme for the IR model on the three noiseless datasets described in \Cref{sec:ModelErrors}. The meaning of the lines and symbols is the same as in \Cref{fig:Error_Ex_X_y} (top), see \Cref{sec:ModelErrors}. The errors $\rho({\bm y}_{\rm v})$ (left, colored, dashed) and $\rho(X_{\rm v})$ (right, colored, dashed) on the training dataset are the mean CV errors (\ref{eq:RhoV}). All data are exact up to numerical precision. The only differences with the IR model of \cref{fig:Error_Ex_X_y} (top) are the over-complete nature of the training datasets ($n=150$) and the application of the LS estimates (\ref{eq:LSTraining}) of the coefficient matrix $\hat{A}_{r}$. Therefore, the horizontal axis in \Cref{fig:Error_Ex_X_y} is the number $r$ of columns in the rectangular basis matrix $V\in{\mathbb R}^{n\times r}$, which is no longer equal to the number of samples $n$.

In the tests of \Cref{fig:Error_Ex_X_y} (bottom-left), the $\rho({\bm y}_{\rm v})$ error starts at the level of machine precision for $r=2$ and increases thereafter due to the accumulation of numerical errors. That the $\rho({\bm y}_{\rm v})$ error grows for $r>2$ has to do with the fact that the exact representation of ${\bm y}$ is attained already at $r=2$ in the monomial basis. 
The enlarged coloured markers in the right plot indicate the minima of the $\rho(X_{\rm v})$, the same markers in the left plot show the levels of $\rho({\bm y}_{\rm t})$ errors attained with the corresponding $r^{*}$. As can be seen, the optimal $r^{*}$ gives the smallest achievable error $\rho({\bm y}_{t})$. Hence, in this noiseless case the regularization procedure correctly identifies the rank of the complete dataset as the optimal $r^{*}$.

In the next numerical experiments, for training, validation and testing, we use a synthetic dataset with $p=202$ predictors. Two of these $202$ predictors are generated by randomly sampling a non-functional relationship such as those shown in \Cref{fig:nn-func-3d-example}. One of these `non-functional' columns has the data from a curve that cannot be parameterized by $y$, another has the data from a cone. The remaining $200$ features are obtained by evaluating polynomials of degree lower than $12$ at the sample points $y_{i}$. While we have investigated all three types of noise listed above Eq.~(\ref{eq:NoiseStructure}), we only present the results for the first and the third cases, as the second, i.i.d. case appeared to be very similar to the third, non-i.i.d case. In all examples we use $n=150$ and the optimal polynomial degree $r^{*}$ is found with the CV method at the minimum of $\rho(X_{\rm v})$.

\Cref{fig:Regularization1} illustrates the application of the regularization procedure to the problem where only the dependent-variable data $\bm y$ contains additive Gaussian noise. The columns of the matrix $X$ are exact up to machine precision. We consider three different standard deviations $\sigma=0.05, 0.1$, and $0.2$, corresponding to $5\%$, $10\%$, and $20\%$ levels of noise relative to the ${\bm y}$-data magnitude. As can be seen from the test-error curves (top-left plot, dim solid gray), the errors $\rho({\bm y}_{\rm t})$ attained with these choices of $r^{*}$ (top-right plot,  enlarged colored squares) are close to the smallest achievable $\rho({\bm y}_{\rm t})$ errors for all considered noise levels. At the same time, the minimal $\rho({X}_{\rm v})$ errors do not correspond to the smallest achievable $\rho({X}_{\rm t})$ errors (top-right plot, dim solid gray curves).

\Cref{fig:Regularization2} illustrates the case of additive noise in the dependent variable and in one-third of the predictors. While the behavior is generally similar to the previous case, the attained optimal $\rho({\bm y}_{\rm t})$ errors increase more rapidly with the level of noise. The growth of the $\rho({\bm y}_{\rm v})$ error (bottom-left, colored dashed) is now caused not only by the accumulation of numerical errors and the statistical noise in ${\bm y}$ (via $V$), but by the statistical noise in the $X$-data as well (via $\hat{A}$).
%\begin{figure}[t!]
%\begin{subfigure}{0.9\linewidth}
%        \includegraphics[width=1\linewidth,trim= {3cm 2cm 2cm 3cm},clip]{Images/Noisy_filtered_y_with_2q_exact_X.pdf}
%\end{subfigure}\vfill
%     \begin{subfigure}{0.9\linewidth}
%            \includegraphics[width=1\linewidth,trim= {3cm 0cm 2cm 3cm},clip]{Images/noisy_filtered_y_with_2q_partially_noisy_X_ls_system.pdf}
%\end{subfigure}
%\caption{
%\label{fig:Regularization}
%Top: application of the regularization procedure in the case of exact predictor data $X$ and noisy dependent-variable data ${\bm y}+{\bm \epsilon}$, ${\bm \epsilon}\sim{\mathcal N}({\bm 0},\sigma_{i}^{2}I)$, with $\sigma_{1}=0.05$ (blue), $\sigma_{2}=0.1$ (orange), and $\sigma_{3}=0.2$ (green). Bright colored lines: validation errors $\rho({\bm y}_{\rm v})$ (left, right vertical axis) and $\rho(X_{\rm v})$ (right). Dim dashed gray lines: errors on the test dataset, $\rho({\bm y}_{\rm t})$ (left, left vertical axis) and $\rho(X_{\rm t})$ (right). Bright colored square markers: values of the test errors $\rho({\bm y}_{\rm t})$ and $\rho({X}_{\rm t})$ attained at $r^{*}$ (left and right respectively). Bright colored diamond markers: values of the test errors $\rho({\bm y}_{\rm t})$ (left, left vertical axis) and $\rho({X}_{\rm t})$ (right) attained with $r^{*}$ and the optimal set $\{p\}_{\text{opt}}$ of predictors. Bottom: same as in the row above, but with a third of the columns of $X$ affected by the additive Gaussian noise of the same type as the noise in the dependent variable ${\bm y}$.}
%\end{figure}

\begin{figure}[t!]
\includegraphics[width=1\linewidth,trim= {3cm 0cm 2cm 3cm},clip]{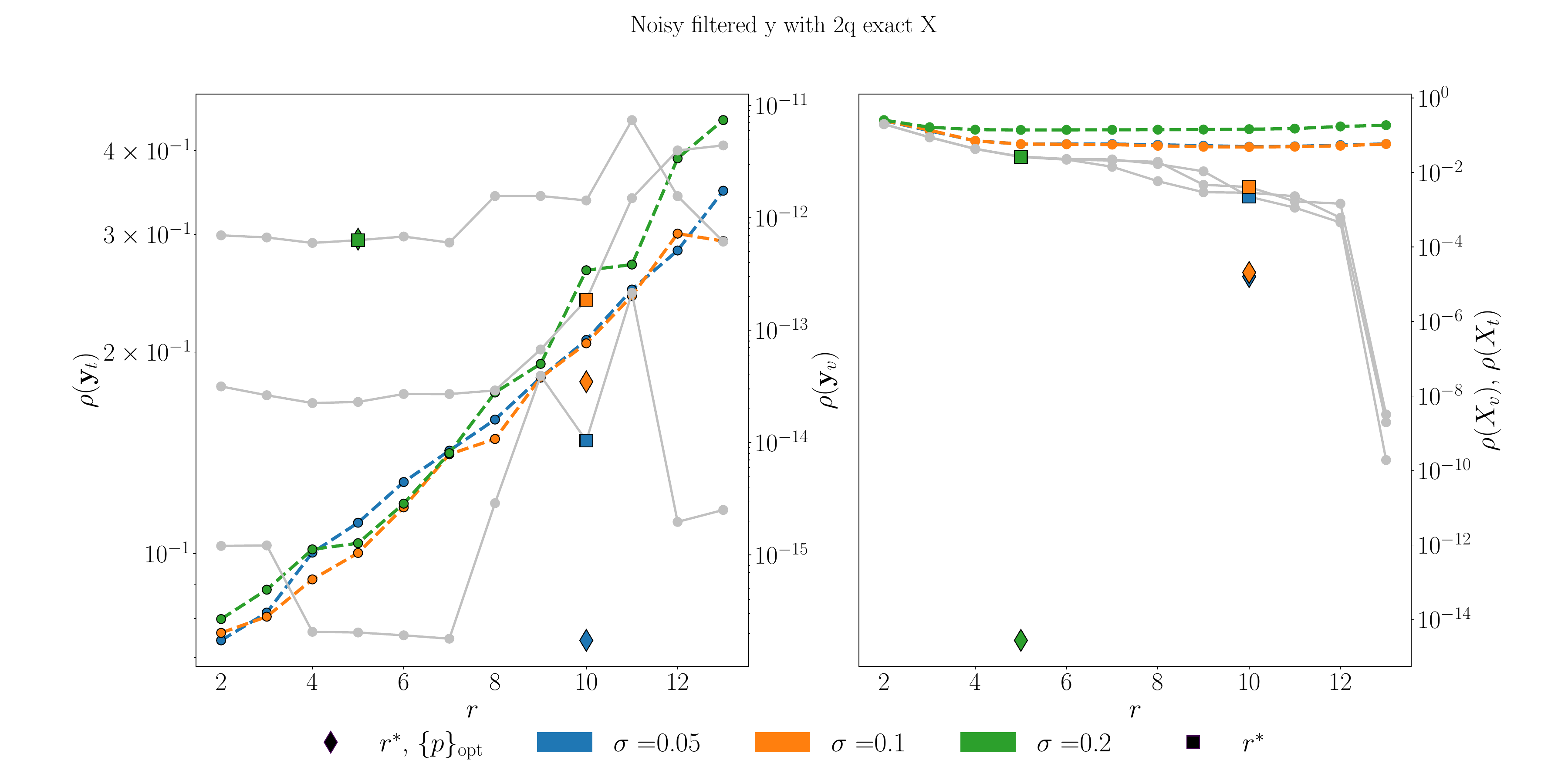}
\caption{
\label{fig:Regularization1}
Regularization and improper-variable removal in the case of noisy dependent-variable data ${\bm y}+{\bm \epsilon}$, ${\bm \epsilon}\sim{\mathcal N}({\bm 0},\sigma_{i}^{2}I)$, with $\sigma_{1}=0.05$ (blue), $\sigma_{2}=0.1$ (orange), and $\sigma_{3}=0.2$ (green). Predictor data $X$ is noiseless. Bright dashed colored lines: validation errors $\rho({\bm y}_{\rm v})$ (left plot, right vertical axis) and $\rho(X_{\rm v})$ (right plot). Dim solid gray lines: errors on the test dataset, $\rho({\bm y}_{\rm t})$ (left plot, left vertical axis) and $\rho(X_{\rm t})$ (right plot). Bright colored square markers: values of the test errors $\rho({\bm y}_{\rm t})$ and $\rho({X}_{\rm t})$ attained at $r^{*}$. Bright colored diamond markers: values of the test errors attained with $r^{*}$ and the optimal set $\{p\}_{\text{opt}}$ of predictors (improper variables removed). 
}
\end{figure}

\begin{figure}[t!]
\includegraphics[width=1\linewidth,trim= {3cm 0cm 2cm 3cm},clip]{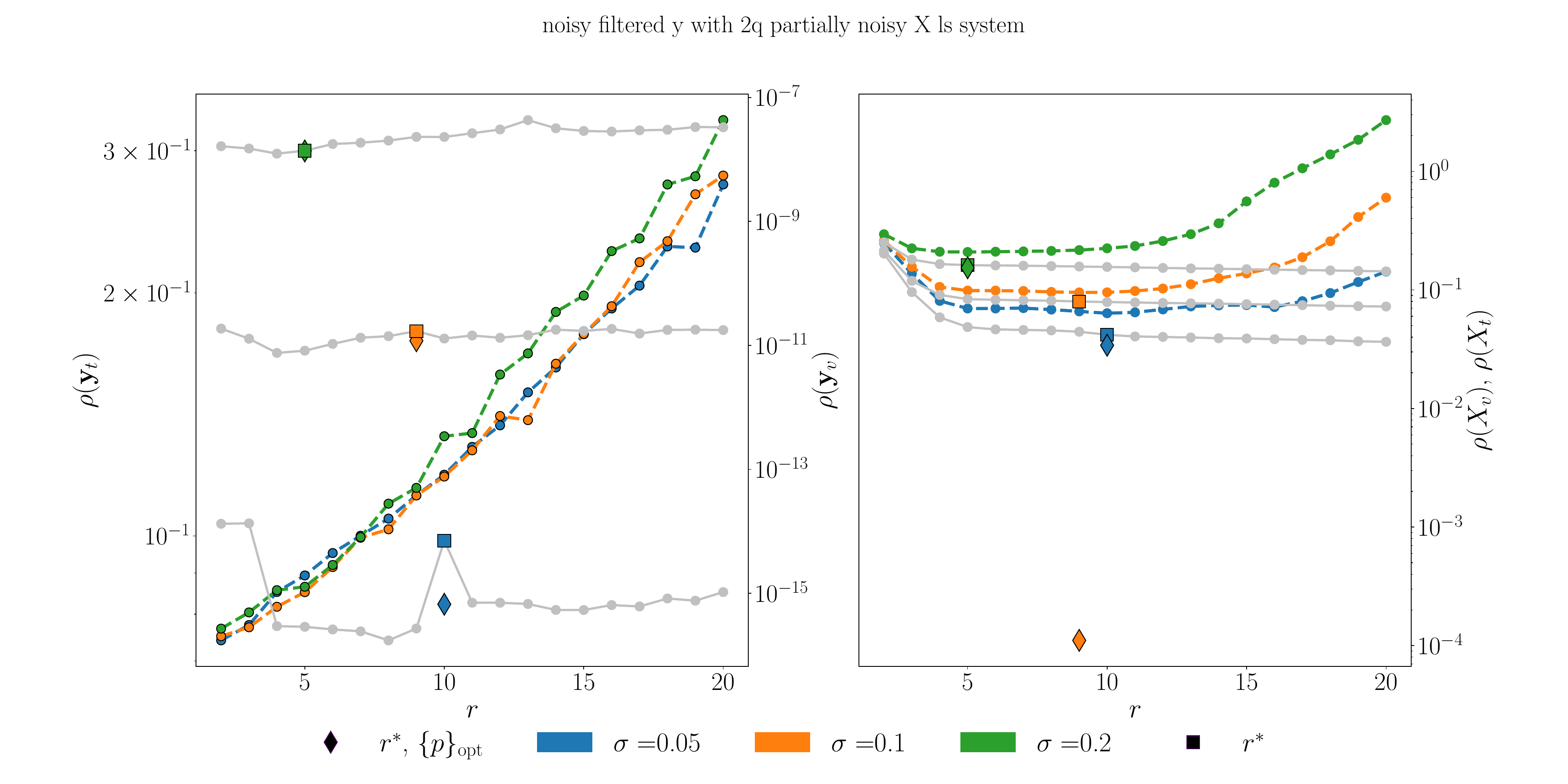}
\caption{
\label{fig:Regularization2}
Regularization and improper-variable removal with additive Gaussian noise present in both the dependent-variable ${\bm y}$ as well as in one third of predictor variables (columns of $X$). Levels of noise and plot legend are the same as in \Cref{fig:Regularization1} 
}
\end{figure}

Finally, we remark that the error $\rho({\bm y}_{\rm v})$ attained with the optimal $r^{*}$ on the training-validation dataset is always much smaller (see the range of the right vertical axis) than the corresponding error $\rho({\bm y}_{\rm t})$ on the test dataset. This is, obviously, one more case of the benign overfitting \cite{BenignOverfitting2020}, \cite{BenignOverfitting2022}, caused by the choice of the monomial column basis in or IR model rather than the statistical properties of the noise.

\section{Removal of improper predictors}
\label{sec:Features}
From the analysis and examples of the previous sections it is clear that a linear model, be it the classical MLR implementation or the present IR formulation, will produce reasonably exact predictions even if some of the predictor variables do not satisfy the model assumptions. In the regularized IR formulation of \Cref{sec:Regularization} the model assumption is that the predictor is a polynomial in the dependent variable $y$ of degree at most $r^{*}-1$, where $r^{*}$ is the rank of the optimally regularized $\hat{A}$. In the case of exact data discussed in \Cref{sec:ModelErrors}, the training set may become complete way before some of the predictor functions have been properly sampled. Success of predictions in the presence of such `improper' predictors gives an illusion of understanding of the underlying natural phenomena \cite{IllusionsPaper}. 

In \Cref{sec:Regularization} we have also investigated the case where the noise is present only in some of the predictors. Such noisy predictors may be negatively affecting the predictive power of the model. 
Thus, there appear to be two good reasons to detect and possibly discard both the high-degree/non-functional and the noisy predictors from the model. Additionally, in biological and agricultural applications, the large number of predictors included in the FOP dataset is often due to a broad (untargeted) experimental search in the absence of any {\it a-priori} information. The practical goal of such studies is to identify a preferably small subset of microbiota, fungi, molecules, metabolites, or genes, allowing for future targeted and less expensive measurements of these predictor variables. 

Discarding `unnecessary' predictors is known as feature or variable selection in statistics and ML. The main idea of feature selection is simple: one can sometimes achieve the same or even better prediction with just a subset of predictors. However, no clear principle for the inclusion or removal of any particular predictor has been put forward so far. Therefore, feature selection methods are either combinatorial or heuristic in the ways they produce the candidate subsets of predictors. Moreover, the criterion for choosing a particular subset is the prediction error on the training dataset. Since this error has already been used to find the optimal regularization parameter, the feature selection methods are often regarded with suspicion in statistics \cite{Izenman2008}, Chapter $5$.

So far we have tuned one hyper-parameter, the optimal degree $r^{*}$, to define the regularized polynomial IR model. We have used the CV procedure on the training $X$-data to determine this hyper-parameter. This leaves the training ${\bm y}$-data and the test $X_{\rm t}$-data available for tuning other hyper-parameters in our model.

Since the goal is to remove the eventual `improper' predictors, it is logical to use the prediction of the $X_{\rm t}$-data as the `usefulness' criterion for each predictor. We introduce the column-wise test set prediction error:
\begin{align}
\label{eq:ColumnwiseResidual}
\chi_{j} = \frac{\Vert \hat{\bm x}_{j} - {\bm x}_{j} \Vert_{2}}{\Vert{\bm x}_{j} \Vert_{2}},\;\;\;j=1,\dots,p,
\end{align}
where ${\bm x}_{j}$ and $\hat{\bm x}_{j}$ are, respectively, the $j$th columns of the test matrix $X_{\rm t}$ and its predictor $\hat{X}_{\rm t}$. One would, naturally, like to remove the predictors with large $\chi_{j}$'s. For example, all predictors with $\chi_{j}\geq \tau$. 

This introduces another hyper-parameter, the threshold $\tau$, which can be tuned utilizing the last available portion of our data -- the training ${\bm y}$-data. Recall that this data could not be used to tune the regularization parameter $r^{*}$, since $\rho({\bm y})$ always grows with $r$. Yet, for a fixed $r=r^{*}$ the error $\rho({\bm y})$ depends on the set of included predictors, and we call it the feature removal error. \Cref{tab:DataUsage} summarizes the data usage by the regularized polynomial IR algorithm with predictor removal for tuning its hyper-parameters.
\begin{table}
    \centering
    \begin{tabular}{|c|c|}
        \hline
        {\bf hyper-parameter} & {\bf data used} \\
        \hline
        optimal degree $r^{*}$ & $\rho(X_{\rm v})$, training $X$-data\\
        \hline
        which predictor to remove & $\chi_{j}$, test $X_{\rm t}$-data\\
        \hline
        optimal threshold $\tau_{\rm opt}$ & $\rho({\bm y})$, training ${\bm y}$-data\\
        \hline
    \end{tabular}
    \caption{Hyper-parameters of the regularized polynomial IR model with predictor removal and the datasets used to tune these hyper-parameters.}
    \label{tab:DataUsage}
\end{table}

We define the optimal threshold $\tau_{\rm opt}$ to be the one for which the feature-removal error $\rho({\bm y})$ with $r=r^{*}$ is minimal. Unfortunately, while $\rho({\bm y})$ normally tends to grow for small $\tau$, since one starts to loose the useful predictors, the overall dependence of $\rho({\bm y})$ on $\tau$ is neither unimodal nor smooth. Meaning, that there may be a set or a range of $\tau_{\rm opt}$ that gives approximately the same $\rho({\bm y})$ at $r=r^{*}$. In such cases we suggest to make a subjective choice of $r^{*}$ favouring the smallest number of retained predictors.

We have applied the predictor removal procedure to the case of $200$ polynomial columns with degrees $q-1\leq 12$ and two non-functional columns in the data-matrix $X$, where the dependent variable and a third of the columns of $X$ are affected by additive Gaussian noise with $\sigma=0.05$.  

\Cref{fig:FeatureRemovalError} presents two examples of the feature removal error on the training ${\bm y}$-data (solid blue) together with the corresponding non-accessible feature removal errors on the test ${\bm y}_{\rm t}$-data (dashed gray). In the left plot, where the level of noise is $\sigma=0.05$, the choice of the optimal threshold $\tau_{\rm opt}$ from the minimum of the training ${\bm y}$-data corresponds to the smallest possible number of predictors. In the right plot, where $\sigma=0.1$, the detected minimum of the feature removal error is situated in the second valley of the error function and corresponds to many unnecessary predictors retained in the model, as the depth of both valleys is approximately the same. Checking for other local minima in this case and choosing the leftmost along the $\tau$-axis would be a better strategy to find $\tau_{\rm opt}$. Regardless, it is comforting to observe that the feature removal errors on the training and test datasets appear synchronous in these numerical examples.
\begin{figure}
\begin{subfigure}[c]{0.49\linewidth}
\includegraphics[width=1\linewidth, trim=2.6cm 0cm 2.5cm 1.92cm, clip]{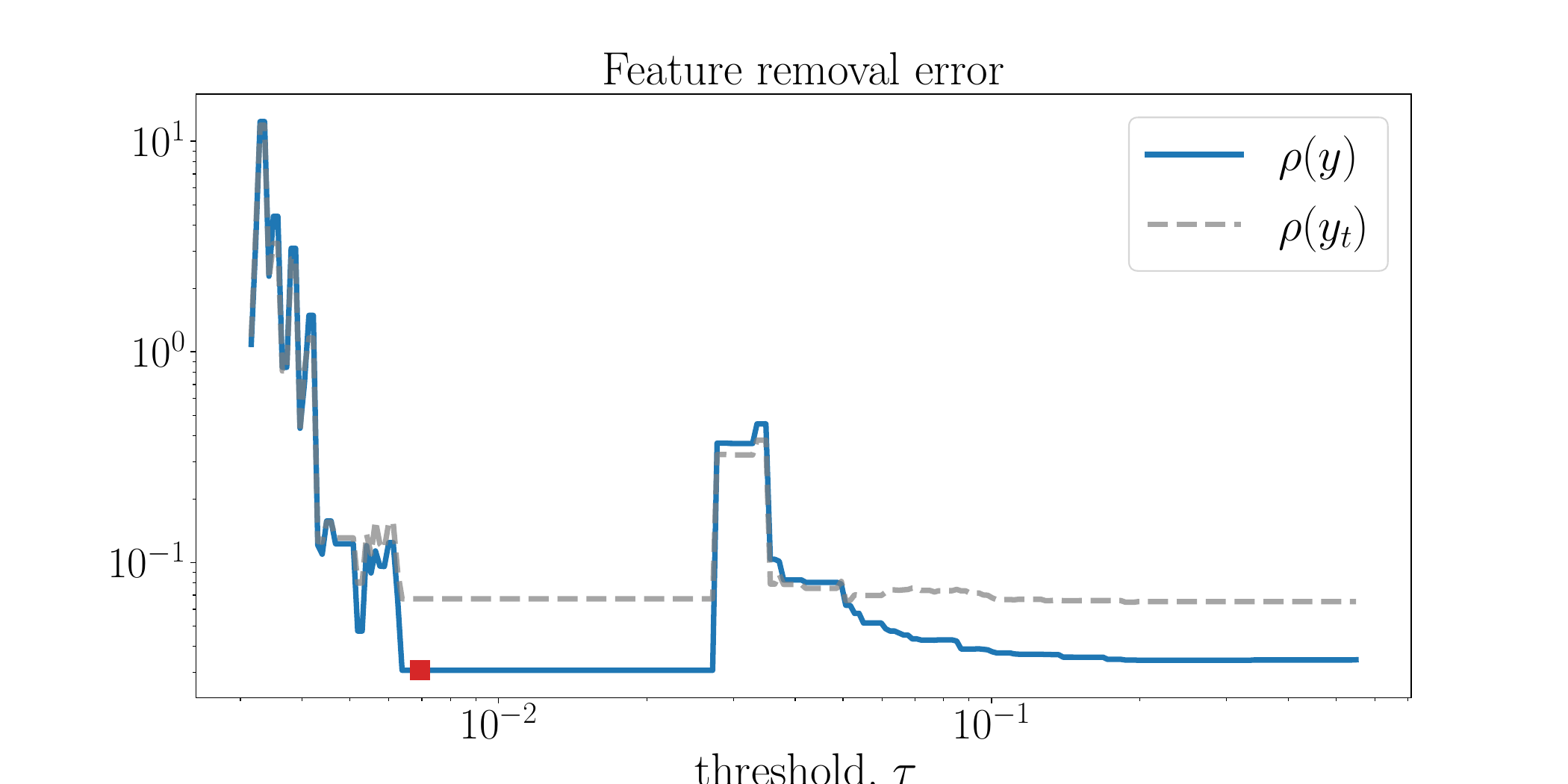}
\end{subfigure}\hfill
 \begin{subfigure}[c]{0.49\linewidth}
     \includegraphics[width=1\linewidth, trim=2.6cm 0cm 2.5cm 1.92cm, clip]{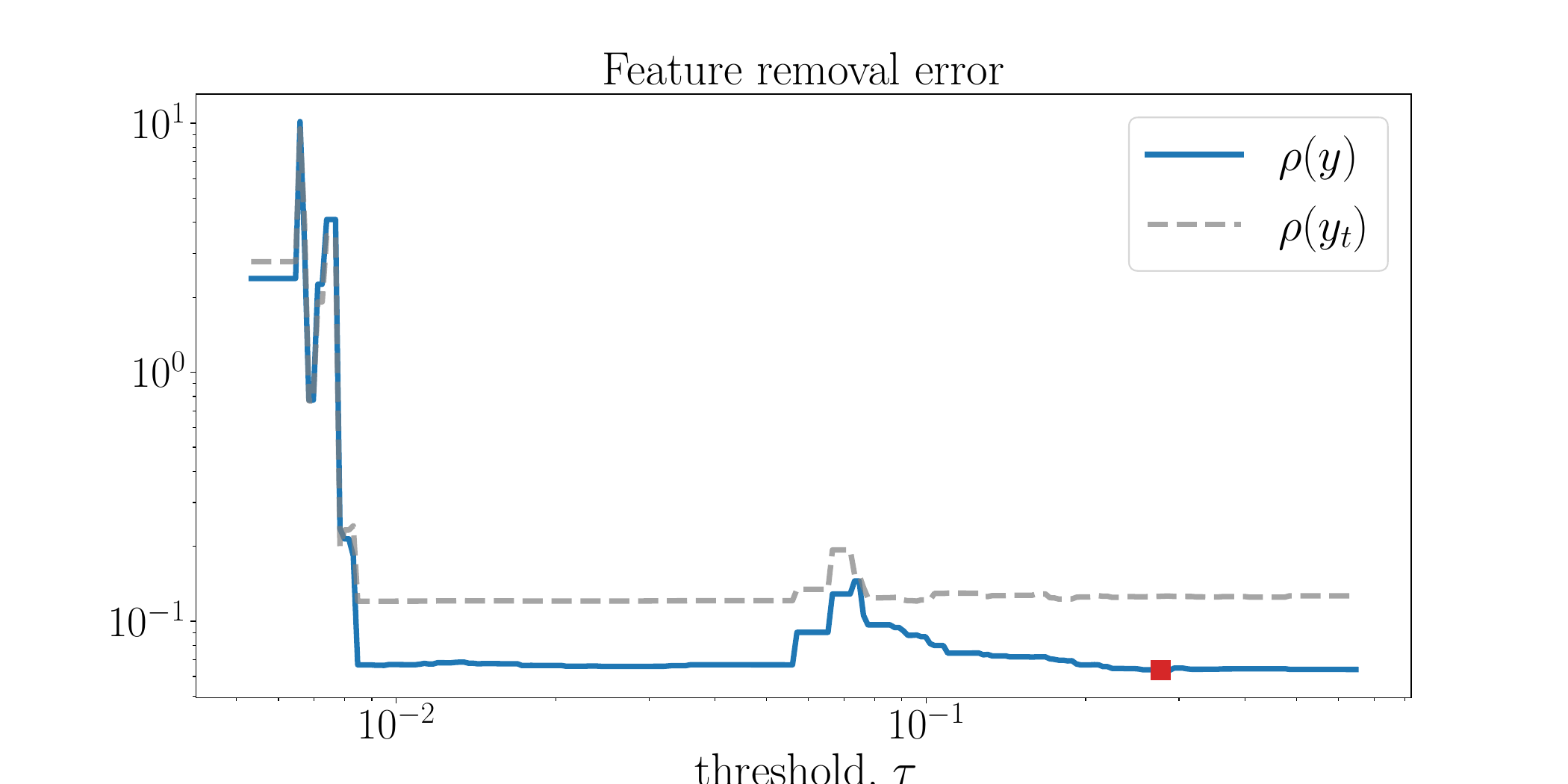}
 \end{subfigure}
 \caption{\label{fig:FeatureRemovalError} Feature removal errors $\rho({\bm y})$ (solid blue) and $\rho({\bm y}_{\rm t})$ (dashed gray), both with $r=r^{*}$, for synthetic datasets containing additive noise in the dependent variable and one-third of predictors, $\sigma=0.05$ (left) and $\sigma=0.1$ (right). Red squares indicate the minima, corresponding to the optimal thresholds $\tau_{\rm opt}$.}
\end{figure}

In \Cref{fig:col_sel_res_X} (top) the green triangles show the true `degrees' (lengths of the nonzero parts of the corresponding columns of the coefficient matrix $A$, used to generate the predictor data) for each predictor. The vertical dashed dark-grey lines indicate the predictors containing the Gaussian noise. The vertical dashed red lines indicate the non-functional predictors. The horizontal solid orange line depicts the optimal degree $r^{*}$.

In \Cref{fig:col_sel_res_X} (bottom) the orange circles show the value of the column prediction error $\rho({\bm x}_{j})$ on the test $X_{\rm t}$-data for each predictor obtained with $r=r^{*}$. The dashed blue line is the optimal threshold $\tau_{\rm opt}$. All predictors with the errors above that line (solid light-gray vertical lines) are discarded. Notice that all noisy predictors, as well as both non-functional predictors are successfully removed.

\begin{figure}
    \centering
     \includegraphics[width=\linewidth, trim=4.2cm 0 4.7cm 3.2cm, clip]{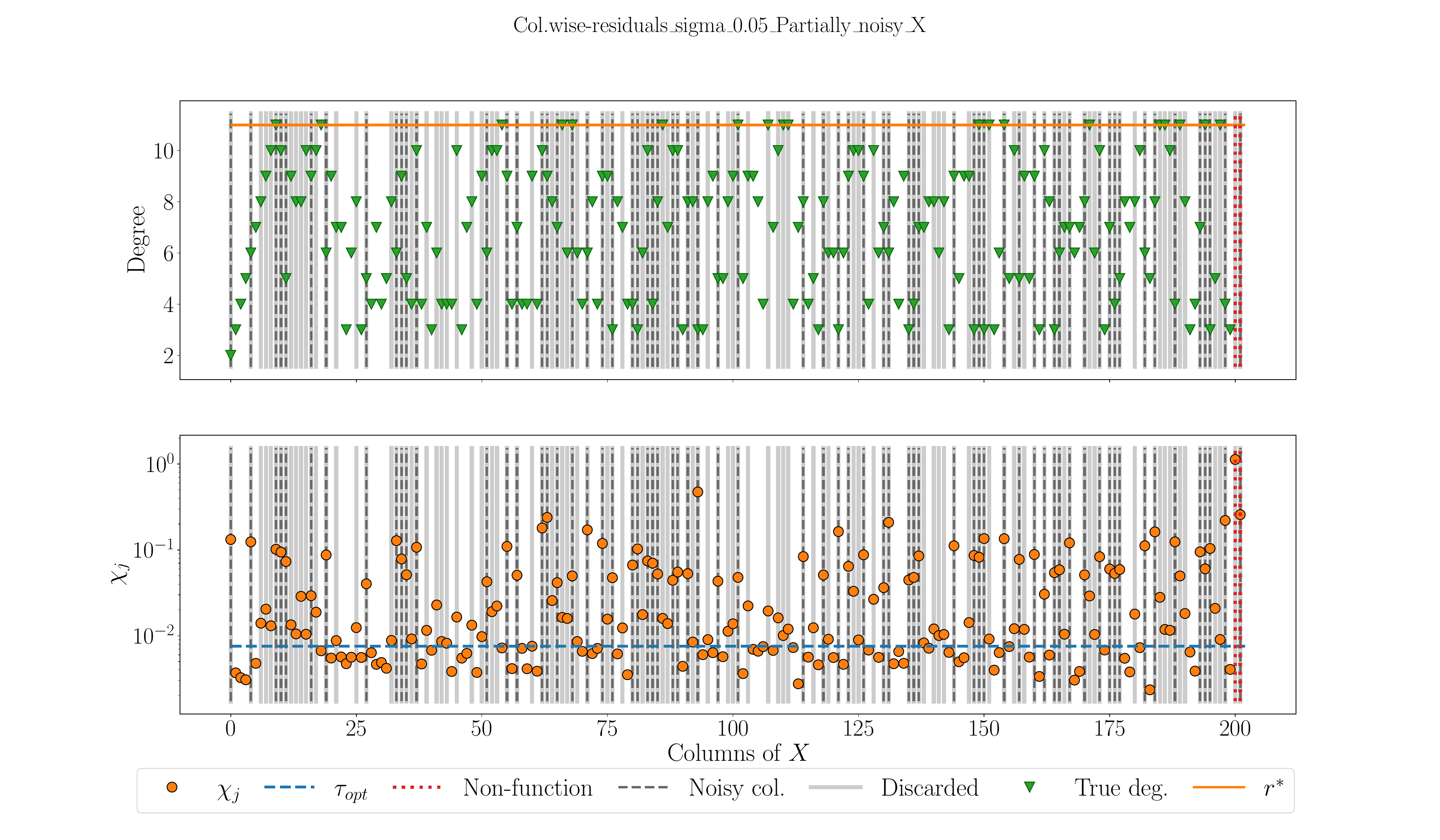}
     \caption{Removal of `improper' predictors (vertical solid light-gray lines) from synthetic data. Top: actual degrees of each predictor (green triangles), the global truncation degree $r^{*}$ (orange horizontal line), predictors containing Gaussian noise (vertical dashed dark-gray lines), non-functional predictors (vertical dotted red lines). Bottom: $\chi_{j}$ errors (orange circles), the optimal threshold $\tau_{\rm opt}$ (horizontal dashed blue line). }
     \label{fig:col_sel_res_X}
\end{figure}

Diamond markers in \Cref{fig:Regularization1} and \Cref{fig:Regularization2} show by how much the removal of `improper' predictors improves the prediction error $\rho({\bm y}_{t})$. In fact, it is close to the minimal achievable error with all predictors retained. That level of error was, however, not always obtained with the present crude global regularization procedure. Hence, the removal of predictors may sometimes compensate for the errors in determining the optimal regularization parameter.

\section{Application to chemometric data}\label{sec:yarn}
In this section we present the application of the regularized IR model with predictor removal to the experimental dataset that is often used to compare various regularized MLR models, such as the Principal Component Regression (PCR) or the Partial Least Squares (PLS) \cite{Izenman2008}. This dataset was first published in \cite{YarnPaper} and is also included in the $R$-library {\tt pls} \cite{plsRpackage} under the name of {\tt yarn}. The dataset consists of $28$ Near InfraRed (NIR) spectra of Polyethylene terephthalate (PET) yarns, measured at $p=268$ wavelengths, and $28$ corresponding yarn densities. Seven of these $28$ samples are designated as the test dataset and the remaining $n=21$ constitute the training dataset. The predictor variables (spectral amplitudes) change continuously with the wavelength, which becomes clear if one plots the rows of $X$ as curves, see the middle plot of \Cref{fig:yarn3}.

The training dataset columns have been sorted by the ${\bm y}$-data. Since $n=21<150$, we use the visual inspection of the fitted polynomials, see \Cref{fig:yarn2}, rather than the CV procedure to determine the optimal degree $r^{*}=5$. The corresponding validation error $\rho(X_{\rm v})$ is shown in \Cref{fig:yarn4} (left plot) with its minimum reached at the maximum displayed degree $r^{*}=5$. The first three rows of the coefficient matrix $\tilde{A}$ obtained with $r=r^{*}$ are shown as curves in the bottom plot of \Cref{fig:yarn3}. These curves correspond to the constant, linear and quadratic components in the polynomial fits $x_{j}(y)$ of each predictor. The column-wise prediction errors $\chi_{j}$ are shown in the top plot of \Cref{fig:yarn2} (solid blue), together with the optimal threshold $\tau_{\rm opt}$ (horizontal dashed red). The $\rho({\bm y})$ error with $r=r^{*}=5$ as a function of $\tau$ can be seen in \Cref{fig:yarn4} (right plot, solid blue), with its minimum indicated by the red square marker. This minimum is close to the minimum of the inaccessible $\rho({\bm y}_{\rm t})$ error (dashed gray). Here, as in the numerical examples of \Cref{sec:Features}, we also observe the synchronous behavior of the feature removal errors on the training and test datasets.

The removed `improper' predictors are indicated with vertical solid light-gray lines in \Cref{fig:yarn2} (middle and bottom). We have also extracted a few typical predictor data (colored vertical lines in \Cref{fig:yarn3}) and displayed them in \Cref{fig:yarn2} using the same colors and line styles for the corresponding polynomial fits. The retained predictors are in the left plot of \Cref{fig:yarn2} and the removed predictors are in the right plot. It is easy to recognize the cone-type data, similar to the data in \Cref{fig:nn-func-3d-example} (right), in one of the discarded predictors (blue). All predictor data in the discarded gray zone around the dashed blue line in \Cref{fig:yarn3} have this `conical' shape. This means that the NIR data in this frequency band is situated on a hyper-cone. Even if this frequency band is not required to build a high-quality predictive IR model, it may become useful for testing other higher-dimensional manifold and nonlinear models. 

Finally, in \Cref{fig:yarn5} we display the standard quality-of-prediction scatter plot which compares the exact values of ${\bm y}$ against their predictions $\hat{\bm y}$. As can be seen, the removal of the `improper' predictors significantly improves the quality of prediction. In fact, the prediction error $\rho({\bm y}_{\rm t})$ on the test dataset goes down from $0.28$ to $0.05$ after the predictor removal procedure.

\begin{figure}
\begin{subfigure}[t]{0.49\linewidth}
    \includegraphics[width=\linewidth]{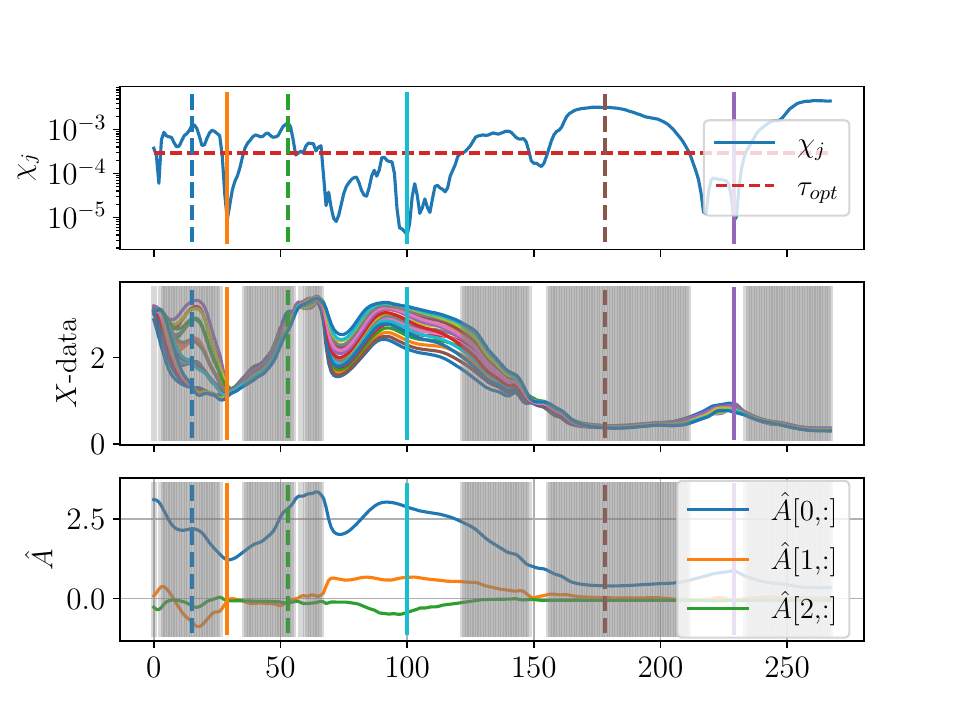}
\caption{}
    \label{fig:yarn3}
\end{subfigure}
\begin{subfigure}[t]{0.49\linewidth}
    \includegraphics[width=\linewidth, trim={0 0 0 1.4cm}, clip]{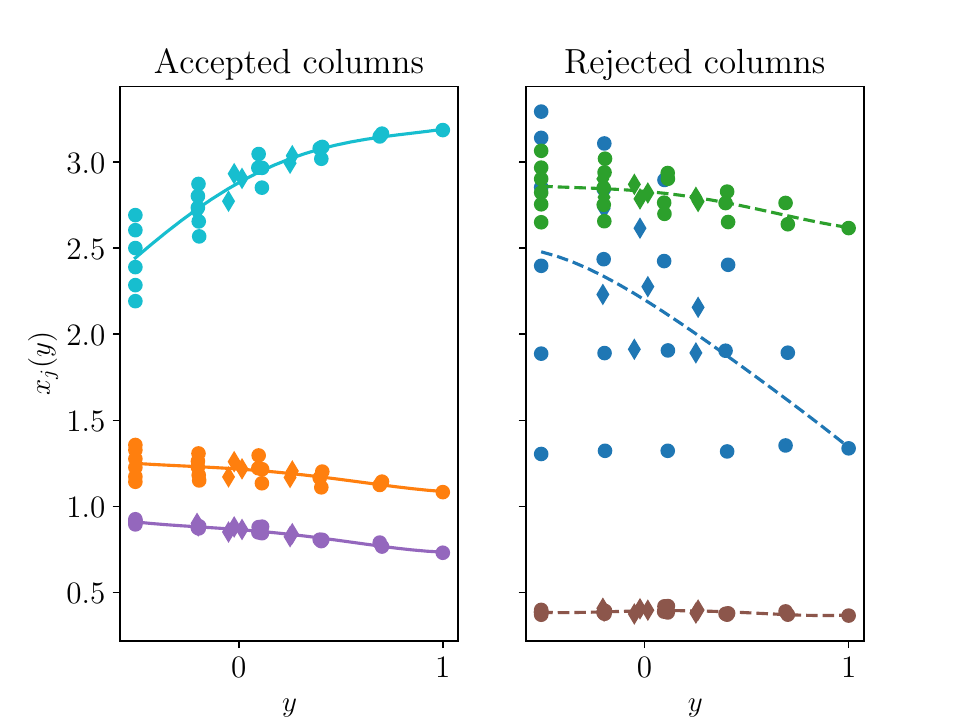}
\caption{}
    \label{fig:yarn2}
\end{subfigure}
    \caption{(a): column-wise prediction errors $\chi_{j}$ and the optimal threshold $\tau_{\rm opt}$ (top); 
    rows of $X$ displayed as curves (middle); first three rows of $\hat{A}$ displayed as curves (bottom). Removed predictors are shown a vertical light-gray lines (middle, bottom). 
    (b): predictor data and the corresponding polynomial fits for the retained (left) and removed (right) predictors.}
\end{figure}

 \begin{figure}
    \centering
\begin{subfigure}[t]{0.48\linewidth}
    \includegraphics[width=\linewidth]{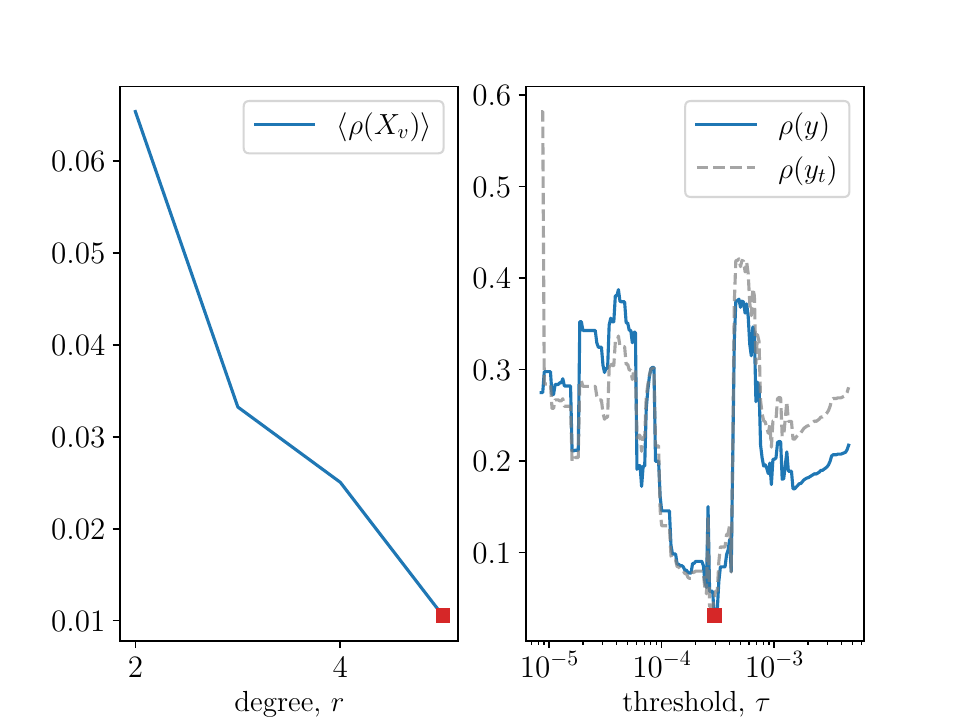}
    \caption{}
    \label{fig:yarn4}
\end{subfigure}\hfill
\begin{subfigure}[t]{0.48\linewidth}
     \includegraphics[width=\linewidth]{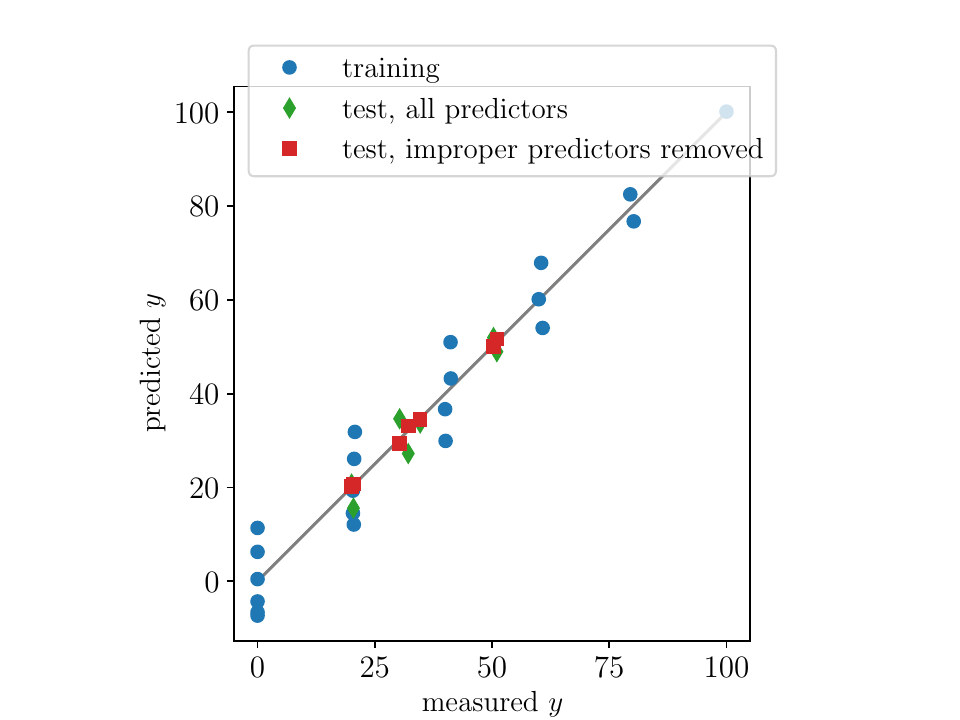}
    \caption{}
    \label{fig:yarn5}   
\end{subfigure}  
\caption{(a): validation error $\rho({\bm y}_{\rm v})$ as a function of $r$ (left); feature removal error $\rho({\bm y})$ at $r=r^{*}$ as a function of threshold $\tau$ (right). (b): $\hat{\bm y}$ v.s. ${\bm y}$ for the training dataset (blue circles) and the test dataset predicted at $r=r^{*}$ with all predictors (green diamonds) and after the removal of `improper' predictors (red squares).}
\end{figure}

\section{Conclusions}
The hyper-curve PARCUR model introduced in this paper provides an alternative, predictor-centered look at the overparameterized multiple linear regression. Within this framework we have been able to focus on the individual roles of predictors and to show that a linear model can be trained and will make perfect predictions even in the presence of predictor variables that violate the linear model assumptions, thus giving an illusion of understanding of the underlying natural phenomena \cite{IllusionsPaper}. The column-centered nature of our approach allowed us to come up with a rigorous algorithm for detecting such `improper' predictors. Moreover, we observe that removing these predictors improves not only the adequacy but also the predictive power of the model.

The polynomial IR version of the PARCUR model has been investigated here in considerable detail. It attempts to build an `inverse' relation between the dependent variable and each predictor that we believe is much easier to interpret than the weights of the regression vector in the classical MLR model. The polynomial IR model appears to work well with chemometric data, but may also be suitable for other `smooth' predictors. The main problem with applying the IR model is the difficulty of sorting the dependent variable by magnitude in the presence of noise. While directly using a noisy dependent variable as a parameter in the PARCUR model is not prohibited, it complicates the interpretation of the regression results.

Certain types of predictor data (microbiome, metabolome, genetic) may exhibit jump-like changes with the dependent variable, especially, in its lower and higher ranges. Successful application of the IR model to these kinds of data will depend on finding a suitable basis for the column function space, e.g., a piecewise-continuous finite-element basis.

Although the PARCUR model allows for a flexible per-column regularization, in this paper we have failed to illustrate its potential advantage over the traditional global regularization approach. A separate investigation of this regularization technique is warranted, since it may provide with a more rigorous way to detect and remove `improper' predictors.

Finally, as we have seen in the application of the polynomial IR model to chemometric data, it can detect the subsets of predictors for which a higher-dimensional model is more appropriate. Nonlinear functional relations between the predictors and the dependent variable produce data that are situated on such higher-dimensional manifolds. A two-parameter hyper-surface model would be a natural extension of the present single-parameter hyper-curve model.

\funding{This work was funded by HZPC Research B.V., Averis Seeds B.V., BO-Akkerbouw, and European Agricultural Fund for Rural Development}

\bibliographystyle{unsrt}
\bibliography{bibl}
%%%%%%%%%%%%%%%%%%%%%%%%%%%%%%%%%%%%%%%%%%%%%%%%%%%%%%%%%%%%%%%%%%%%%%%%%%%%%%%%

\end{document}